\documentclass{article}

\usepackage{microtype}
\usepackage{graphicx}
\usepackage{subcaption}
\usepackage{booktabs} 
\usepackage{threeparttable}  
\usepackage{multirow}
\usepackage[normalem]{ulem}
\usepackage{makecell}
\usepackage{xcolor}
\usepackage[most]{tcolorbox}
\usepackage{array}

\newcommand{\pref}{\textbf{\large$\checkmark$}}
\newcommand{\nopref}{\textcolor{gray}{\large$\times$}}

\newcommand{\cNF}[1]{\textcolor{blue!70!black}{#1}}       
\newcommand{\cCN}[1]{\textcolor{orange!80!black}{#1}}     
\newcommand{\cST}[1]{\textcolor{green!60!black}{#1}}      
\newcommand{\cFL}[1]{\textcolor{purple!70!black}{#1}}     
\newcommand{\cSTN}[1]{\textcolor{red!70!black}{#1}}       
\newcommand{\cellwrap}[1]{\parbox[t]{\linewidth}{\raggedright #1}}
    
\newcommand{\mLong}[1]{\textcolor{teal!70!black}{#1}}       
\newcommand{\mShort}[1]{\textcolor{gray!80!black}{#1}}      
\newcommand{\mInter}[1]{\textcolor{cyan!70!black}{#1}}      
\newcommand{\mInfo}[1]{\textcolor{orange!80!black}{#1}}     

\usepackage{hyperref}

\usepackage[ruled,vlined]{algorithm2e}

 \usepackage[preprint]{icml2026}

\usepackage{amsmath}
\usepackage{amssymb}
\usepackage{mathtools}
\usepackage{amsthm}

\usepackage[capitalize,noabbrev]{cleveref}

\theoremstyle{plain}
\newtheorem{theorem}{Theorem}[section]
\newtheorem{proposition}[theorem]{Proposition}

\theoremstyle{definition}
\newtheorem{definition}[theorem]{Definition}

\theoremstyle{remark}

\usepackage{booktabs}
\usepackage{tabularx}
\usepackage{xcolor}
\usepackage{tcolorbox}
\tcbset{sharp corners, boxrule=0.6pt, colback=white}

\newcommand{\mono}[1]{{\ttfamily #1}}

\newcommand{\win}[1]{\textcolor{red!70!black}{#1}}
\newcommand{\lose}[1]{\textcolor{blue!70!black}{#1}}

\newcommand{\CaseStudyTemplate}[6]{%
\begin{tcolorbox}[width=\linewidth, colframe=black]
\small
\textbf{User Query:} \mono{#1}\par
\vspace{0.35em}
\textbf{Selected Attributes $\mathcal{A}$:} #2\par
\vspace{0.6em}

\begin{tabularx}{\linewidth}{@{}p{0.18\linewidth}X@{}}
\toprule
\textbf{Original Prompt} & \mono{#3} \\
\midrule
\textbf{Modified Prompt} &
\begin{tcolorbox}[colframe=black!35, colback=black!3, boxrule=0.4pt, left=1mm, right=1mm, top=0.8mm, bottom=0.8mm]
\mono{#4}
\end{tcolorbox}
\\
\bottomrule
\end{tabularx}

\vspace{0.35em}
\textbf{Win Response:} \win{#5}\par
\vspace{0.35em}
\textbf{Lose Response:} \lose{#6}\par
\end{tcolorbox}
}

\usepackage[textsize=tiny]{todonotes}

\icmltitlerunning{EXACT: Explicit Attribute-Guided Decoding-Time Personalization}

\begin{document}

\twocolumn[
  \icmltitle{EXACT: Explicit Attribute-Guided Decoding-Time Personalization}

  \icmlsetsymbol{equal}{*}

  \begin{icmlauthorlist}
    \icmlauthor{Xin Yu}{yyy}
    \icmlauthor{Hanwen Xing}{comp}
    \icmlauthor{Lingzhou Xue}{yyy}

  \end{icmlauthorlist}

  \icmlaffiliation{yyy}{Department of Statistics, the Pennsylvania State University, PA, USA}
  \icmlaffiliation{comp}{University of Southern California, CA, USA}

  \icmlcorrespondingauthor{Lingzhou Xue}{lzxue@psu.edu}

  \icmlkeywords{Machine Learning, ICML}

  \vskip 0.3in
]



\printAffiliationsAndNotice{}  

\begin{abstract}

Achieving personalized alignment requires adapting large language models to each user’s evolving context. While decoding-time personalization offers a scalable alternative to training-time methods, existing methods largely rely on implicit, less interpretable preference representations and impose a rigid, context-agnostic user representation, failing to account for how preferences shift across prompts. We introduce EXACT, a new decoding-time personalization that aligns generation with limited pairwise preference feedback using a predefined set of interpretable attributes. EXACT first identifies user-specific attribute subsets by maximizing the likelihood of preferred responses in the offline stage. Then, for online inference, EXACT retrieves the most semantically relevant attributes for an incoming prompt and injects them into the context to steer generation. We establish theoretical approximation guarantees for the proposed algorithm under mild assumptions, and provably show that our similarity-based retrieval mechanism effectively mitigates contextual preference shifts, adapting to disparate tasks without pooling conflicting preferences. Extensive experiments on human-annotated preference datasets demonstrate that EXACT consistently outperforms strong baselines, including preference modeling accuracy and personalized generation quality.

\end{abstract}

\section{Introduction}
\label{sec: Introduction}
\vspace{-1mm}
Large language models (LLMs) have become foundational to a wide range of real-world applications, driven in part by advances in reinforcement learning from human feedback (RLHF) \citep{rafailov2023direct,ziegler2019fine}. Conventional RLHF pipelines aim to align models with population-level preferences by aggregating large-scale feedback from diverse users. While effective for general alignment, this paradigm overlooks substantial heterogeneity in individual user preferences, motivating growing interest in personalized alignment methods \citep{xu2025personalized,liu2025survey,bai2025towards}. Personalized LLMs seek to produce user-specific outputs while preserving contextual coherence, enabling applications such as personalized recommendation systems \citep{lyu2024llm,zhang2024prospect,yu2025thought}, adaptive dialogue agents \citep{wu2025interpersonal}, and customized content generation platforms \citep{mysore2024pearl}.

Within personalized generation, decoding-time alignment offers a compelling alternative to training-time methods ~\citep{zhou2024beyond,wu2023fine}, which avoid parameter updates and are well-suited to settings with sparse user preference data~\citep{chen2025pad,kim2025drift}. Existing decoding-time personalization approaches generally fall into two categories. One line of work implicitly encodes user preferences through curated prompts or instruction templates~\citep{wang2023cue,richardson2023integrating}; while simple to deploy, this relies on heuristic prompt design and offers limited interpretability of the underlying preferences. The other line of work adopts inference-time optimization, such as PAD~\citep{chen2025pad} and Drift~\citep{kim2025drift}, which explicitly modulates token generation probabilities by mixing logits from multiple models. However, these strategies often incur significant deployment overhead and rely on the assumption of a single preference profile for each user.

Crucially, the foundational theory of constructive preferences in behavioral decision making \cite{slovic1995construction} suggests that preferences are not stored as static traits, but are constructed contextually depending on the task and context. Consistent with this view, Table~\ref{tab:merged_case_preference_reversal} shows that within existing human-preference datasets, a single user often exhibits divergent style preferences across different prompts, a phenomenon we term \emph{contextual preference shifts}. This observation challenges the common assumption in the decoding-time personalization literature that a user can be represented by a single preference profile. Despite this, most existing methods continue to optimize context-agnostic user signals, failing to account for such prompt-dependent variation.

\begin{table*}[t]
\centering
\small
\setlength{\tabcolsep}{5pt}
\renewcommand{\arraystretch}{1.12}

\resizebox{\textwidth}{!}{%
\begin{tabular}{
  >{\raggedright\arraybackslash}p{0.19\linewidth}
  >{\raggedright\arraybackslash}p{0.26\linewidth}
  >{\raggedright\arraybackslash}p{0.25\linewidth}
  >{\raggedright\arraybackslash}p{0.25\linewidth}
  >{\raggedright\arraybackslash}p{0.05\linewidth}
}
\toprule
\textbf{Dataset (User)} & \textbf{Prompt (Title)} & \textbf{Response A} & \textbf{Response B} & \textbf{Style} \\
\midrule

\multirow{2}{*}{\cellwrap{\textbf{PRISM}\\\texttt{RgH765...}}}
& \cellwrap{\textbf{P-A:} \textit{Me [26M] with my fiancee... wedding invitations}}
& \cellwrap{\pref\; \textit{``Her parents printed the wedding invitation card and want her to hide it from me.''}
{\footnotesize\;\textit{(\cNF{neutral} \cNF{restatement})}}}
& \cellwrap{\nopref\; \textit{``...hide the wedding card... because of my in-laws' \cSTN{hate towards me}.''}}
& \cellwrap{\textbf{\cNF{Neutral}}\\[-1pt]\textbf{\cNF{factual}}} \\

\midrule

\multirow{2}{*}{\cellwrap{\textbf{PRISM}\\\texttt{RgH765...}}}
& \cellwrap{\textbf{P-B:} \textit{Insurance company's lawyer seeking reimbursement?}}
& \cellwrap{\pref\; \textit{``Neighbor wants money... previous owner caused damage... neighbor's insurance paid... now wants grandma to pay back. Our insurance is \cSTN{basically MIA}.''}
{\footnotesize\;\textit{(\cCN{context-rich} \cCN{narrative})}}}
& \cellwrap{\nopref\; \textit{``Neighbor's insurer asked for reimbursement... wants \$1600...''}
{\footnotesize\;\textit{(\cNF{minimal facts})}}}
& \cellwrap{\textbf{\cCN{Context}}\\[-1pt]\textbf{\cCN{rich}}\\
[-1pt]\textbf{\cCN{narrative}}} \\

\midrule

\multirow{2}{*}{\cellwrap{\textbf{Summarize from}\\\textbf{Human Feedback}\\\texttt{user5}}}
& \cellwrap{\textbf{T1:} \textit{I want to cook lasagne...}}
& \cellwrap{\nopref\; \textit{``There are many great recipes... it depends on personal preference.''}}
& \cellwrap{\pref\; \textit{``\cST{Ingredients:} ... \cST{Instructions:} 1. ... 2. ...''}}
& \cellwrap{\textbf{\cST{Structured}}\\[-1pt]\textbf{\cST{step-based}}} \\

\midrule

\multirow{2}{*}{\cellwrap{\textbf{Summarize from}\\\textbf{Human Feedback}\\\texttt{user5}}}
& \cellwrap{\textbf{T2:} \textit{Last-minute cramming tips?}}
& \cellwrap{\pref\; \textit{``Prioritize key topics, practice active recall, use mnemonics, take breaks, and sleep.''}
{\footnotesize\;\textit{(\cFL{flowing} \cFL{narrative})}}}
& \cellwrap{\nopref\; \textit{``\cST{1.} Focus on core concepts... \cST{2.} Use mnemonics...''}}
& \cellwrap{\textbf{\cFL{Flowing}}\\[-1pt]\textbf{\cFL{narrative}}} \\

\bottomrule
\end{tabular}
}

\caption{
\textbf{Contextual preference shifts within individual users.}
In \textbf{PRISM} (user \texttt{RgH765...}), the preferred style shifts from \cNF{neutral factual restatement} (P-A) to \cCN{context-rich narrative} (P-B). In \textbf{Summarize from Human Feedback} (user \texttt{user5}), the preference similarly diverges from \cST{structured step-based} formatting (T1) to \cFL{flowing narrative} advice (T2). Colored text highlights salient stylistic attributes.
}\label{tab:merged_case_preference_reversal}
\end{table*}

In this work, we propose \textbf{EXACT} (\textbf{EX}plicit \textbf{A}ttribute-guided de\textbf{C}oding-\textbf{T}ime personalization), a novel decoding-time personalization method driven by RLHF. Leveraging a large attribute set that captures diverse aspects of user preferences, EXACT identifies an optimal user-specific subset of attributes from pairwise preference feedback, maximizing the likelihood of preferred responses in the offline indexing stage. This explicit attribute representation provides a transparent and interpretable characterization of user preferences, distinguishing it from purely implicit prompt-engineering techniques. At online inference time, to guide efficient personalized generation, EXACT retrieves the most relevant attribute subset from the user’s history based on the semantic similarity of the incoming prompt and injects it into the prompt generation context. This design is highly efficient, as it requires only a single base model to be stored and utilized, eliminating the computational and storage overhead of auxiliary attribute models. Finally, our theoretical analysis shows that this similarity-based retrieval mechanism effectively mitigates contextual preference shifts by ensuring the active attribute set adapts to the semantic requirements of each prompt.

\vspace{-1mm}

Our contributions are summarized as follows:

\vspace{-1mm}
\begin{itemize}
     \item  We are the first to identify and formalize \emph{contextual preference shifts} in the decoding-time personalization literature, revealing that individual users can exhibit divergent preferences across different prompts or topics.

   \vspace{-1mm}
    \item We propose \textsc{EXACT}, a decoding-time personalization that learns explicit preference attributes from pairwise feedback. By retrieving and injecting context-relevant attributes at inference, \textsc{EXACT} efficiently mitigates preference shifts while preserving interpretability.
    
   \vspace{-1mm}
    \item We validate the effectiveness of EXACT through extensive experiments across diverse personalized text generation tasks. EXACT consistently achieves superior personalized alignment with user preferences compared to strong decoding-time baselines, particularly in settings with heterogeneous prompt contexts.
\end{itemize}

\vspace{-1mm}
\section{Related Work}

\textbf{Large language model (LLM) alignment.}
LLM alignment aims to make the LLM behavior consistent with human preferences, often using reinforcement learning from human feedback (RLHF) \citep{christiano2017deep, bai2022training}.
Broadly, existing approaches fall into two categories.
(1) RLHF with an explicit reward model: a reward model is trained from human feedback, and then policy optimization, typically Proximal Policy Optimization (PPO) \citep{schulman2017proximal}, is applied to obtain an aligned policy.
In contrast, \emph{decoding-time alignment} avoids expensive RL training by steering generation at inference time \citep{mudgal2024controlled,khanov2024args, han2024value, liu2024decoding, huang2025deal}.
(2) Direct preference optimization: methods such as Direct Preference Optimization (DPO) \citep{rafailov2023direct} train the policy directly from preference data, eliminating the need to fit a separate reward model.

\textbf{Training-time personalization.}
Most existing personalization methods operate at training time, leveraging user-specific data to adapt language models via parameter-efficient fine-tuning (PEFT) \citep{zhang2024personalized,zhu2024lifelong} or reinforcement learning (RL)--based alignment \citep{bose2025lore}.
Within the PEFT paradigm, some approaches learn a \emph{shared} set of lightweight parameters to generalize across diverse user preferences \citep{wozniak2024personalized,kong2024customizing}, while others prioritize \emph{individualized} adaptation by maintaining isolated parameter sets per user.
The latter can enhance personalization and preserve privacy by avoiding cross-user data leakage \citep{bose2025lore,zhong2021useradapter,peng2024pocketllm,tan2024democratizing}.
Beyond PEFT, RL-based methods align models with user preferences through reward-driven optimization.
Recent work further formulates personalization as a multi-objective reinforcement learning (MORL) problem: for example, MORLHF \citep{wu2023fine} and MODPO \citep{zhou2024beyond} train separate reward models for different objectives and combine them during optimization.
However, these training-time approaches typically incur substantial computational overhead.

\textbf{Decoding-time personalization.}
Decoding-time alignment avoids parameter updates and is thus attractive when only a handful of user preference examples are available \citep{chen2025pad,kim2025drift}.
Existing methods fall into two lines: (i) prompt-based approaches encode preferences via curated prompts or templates \citep{wang2023cue,richardson2023integrating}, which are easy to deploy but heuristic and less interpretable; and (ii) logit-steering adaptation methods explicitly adjust token probabilities during generation \citep{gao2025linear, li2023contrastive}, e.g., CoS \citep{he2025context}, CoSteer \citep{lv2025costeer}, and Amulet \citep{zhang2025amulet}.
While effective, logit-steering typically introduces extra inference overhead (especially with many attributes, as in Drift \citep{kim2025drift}) and may require transmitting personal signals to the model, raising privacy concerns and complicating deployment \citep{lv2025costeer}.

\section{Preliminaries}\label{sec_prelimary}

This section reviews the standard RLHF pipeline to motivate our optimization mechanism. For each prompt \(x\), we consider a pair of human-annotated responses \((y_w, y_l)\), where \(y_w\) is a preferred response and \(y_l\) is a dispreferred response.

\vspace{-2mm}

\paragraph{Human preference distribution \(p^*\).}
User preferences are assumed to be generated by an underlying latent reward function \(r^*(x,y)\), which is not directly observable. Under the Bradley--Terry (BT) model, the human preference distribution \(p^*\) over a pair of responses can be expressed as
\begin{equation}
p^{*}(y_w \succ y_l \mid x)
=
\frac{\exp\big(r^{*}(x, y_w)\big)}
{\exp\big(r^{*}(x, y_w)\big) + \exp\big(r^{*}(x, y_l)\big)} .
\end{equation}

\paragraph{KL-regularized reinforcement learning.}
To align a policy \(\pi_\theta\) with human preferences, RLHF maximizes the expected reward while penalizing deviation from a reference (or base) model \(\pi_{\text{base}}\) via a KL regularization term:
\begin{equation}
\label{eq:kl_reg_rl}
\max_{\pi_\theta}\;
\mathbb{E}_{y \sim \pi_\theta(\cdot \mid x)}
\Big[
r(x,y)
-
\beta \, \mathrm{KL}\!\big(\pi_\theta \,\|\, \pi_{\text{base}}\big)
\Big],
\end{equation}
where \(\beta > 0\) controls the strength of the regularization. The optimal solution to Eq.~\eqref{eq:kl_reg_rl} admits a closed-form expression:
\begin{equation}
\label{eq:optimal_pi}
\pi_r(y \mid x)
=
\frac{1}{Z(x)} \, \pi_{\mathrm{base}}(y \mid x)
\exp\!\left( \frac{1}{\beta} \, r(x,y) \right),
\end{equation}
where
\begin{equation}
Z(x)
=
\sum_{y}
\pi_{\mathrm{base}}(y \mid x)
\exp\!\left( \frac{1}{\beta} \, r(x,y) \right),
\end{equation}
is the partition function (see Appendix~\ref{app:kl_closed_form} for the proof). Rearranging Eq.~\eqref{eq:optimal_pi} yields an equivalent expression for the reward function:
\begin{equation}
\label{eq:reward_recovery}
r(x,y)
=
\beta \log \frac{\pi_r(y \mid x)}{\pi_{\mathrm{base}}(y \mid x)}
+
\beta \log Z(x).
\end{equation}

Direct Preference Optimization (DPO) \citep{rafailov2023direct} interprets Eq.~\eqref{eq:reward_recovery} as a bridge that eliminates the need to explicitly optimize a reward model, directly linking the latent human preference distribution to the optimal RLHF policy \(\pi^*\). Substituting Eq.~\eqref{eq:reward_recovery} into the Bradley--Terry model yields
\begin{equation}
\resizebox{0.95\linewidth}{!}{$
p^{*}(y_1 \succ y_2 \mid x)
=
\left[1 + \exp\!\left(
\beta \log \frac{\pi^{*}(y_2 \mid x)}{\pi_{\mathrm{base}}(y_2 \mid x)}
-
\beta \log \frac{\pi^{*}(y_1 \mid x)}{\pi_{\mathrm{base}}(y_1 \mid x)}
\right)\right]^{-1} .
$}
\end{equation}

More importantly, Eq.~\eqref{eq:reward_recovery} provides an explicit quantitative measure of the improvement of any modified policy relative to the reference model \(\pi_{\mathrm{base}}\), including modifications induced purely at inference time, such as changes in the prompting strategy. 
Concretely, given a prompt-level modification \(m\) that induces a new conditional policy 
\(\pi^{m}(y \mid x) := \pi_{\mathrm{base}}(y \mid x, m)\),
the corresponding implicit reward can be written as
\begin{equation}\label{promptreward}
r^{m}(x,y)
=
\beta \log \frac{\pi^{m}(y \mid x)}{\pi_{\mathrm{base}}(y \mid x)}
+ \beta \log Z^{m}(x),
\end{equation}
where \(Z^{m}(x)\) is the induced partition function.

\section{Methodology}
\label{sec:method}

Our method follows a two-stage pipeline:  Offline Indexing in Section~\ref{sec:offline_indexing}, and Online Inference in Section~\ref{sec:online_inference}.

\subsection{Learning user preferences via preference attributes}\label{sec:offline_indexing}

We study \emph{decoding-time} personalization under a fixed base large language model (LLM) $\pi_{\text{base}}$.
Let $x$ denote a user prompt and $\mathcal{A}=\{a_1,\dots,a_K\}$ be a predefined set of semantic attributes, following the attribute design in Drift~\cite{kim2025drift}. Each attribute captures a distinct aspect of user preference or stylistic intent (e.g., conciseness, vividness, formality) (see Table~\ref{tab_attribute} for the full list).
For convenience, we group these attributes into four coarse categories: \emph{style} (e.g., \textit{Formal}, \textit{Concise}), \emph{tone} (e.g., \textit{Direct}, \textit{Empathetic}), \emph{expertise} (e.g., \textit{Analytic}, \textit{Code}), and \emph{values} (e.g., \textit{Principled}, \textit{Utilitarian}).
Attributes are not mutually exclusive and can be correlated (e.g., \textit{Direct} often co-occurs with \textit{Concise}).
We model each user by an \emph{unknown} subset of attributes that governs their preference over generated responses.

Given a prompt $x$, a user (or an annotator acting on the user's behalf~\citep{lee2024aligning}) provides a preferred response $y_w$ and a dispreferred response $y_l$.
Following standard RLHF, we assume these preferences arise from a latent reward function $r^*(x,y)$, which induces a human preference distribution $p^*(y_w \succ y_l \mid x)$ (see Section~\ref{sec_prelimary}).

Our goal is to infer user-specific preference attributes \emph{at decoding-time} and use them to steer personalized generation \emph{without} updating model parameters.

\paragraph{Attribute-guided prompt modification.}

Let \(\pi(\cdot)\) denote the base language model.
Given a set of selected attributes \(A \subseteq \mathcal{A}\), we construct an attribute-augmented prompt by concatenating the original prompt with natural-language descriptions of the attributes,
\[
x^{A} = (x, A),
\]
where we modify the user prompt by appending an explicit attribute block. Specifically, after the original user query, we add a new line that starts with the keyword “Attributes:”, followed by a comma-separated list of the selected attribute subset (e.g., Direct, Concise, Analytic). Figures~\ref{fig:case_prism_1}--\ref{fig:case_sff_3} show the exact template used in our studies.
This yields a modified conditional policy
\begin{equation}
\pi^{A}(y \mid x)
\;:=\;
\pi(y \mid x^{A}) = \pi(y \mid x, A).
\end{equation}

Importantly, this modification is performed purely at inference time and does not alter the model parameters.

\paragraph{Optimization objective from human preferences.}
Applying Eq. (\ref{promptreward}) to the attribute-guided policy \(\pi^{A}\), we obtain the induced reward
\begin{equation}
r^{A}(x,y)
=
\beta \log \frac{\pi^{A}(y \mid x)}{\pi(y \mid x)}
+ \beta \log Z^{A}(x).
\end{equation}

Substituting into the Bradley-Terry model yields

\begin{align}
\begin{split}
& p^{*}(y_w \succ y_l \mid x)\\
&=
\frac{1}
{1 + \exp\!\left(
\beta \log \frac{\pi^{A}(y_l \mid x)}{\pi(y_l \mid x)}
-
\beta \log \frac{\pi^{A}(y_w \mid x)}{\pi(y_w \mid x)}
\right)} \\
& \propto
\log \pi^{A}(y_w \mid x)
-
\log \pi^{A}(y_l \mid x),
\end{split}
\end{align}

where the last line follows from the monotonicity of the logistic function, 
and we only retain the terms that depend on the induced policy $\pi^{A}$.
Constants and reference-policy terms that do not affect the ordering are omitted.
Therefore, learning user-specific personalization reduces to selecting an attribute
set $A$ such that the induced policy $\pi^{A}$ maximizes the likelihood of the
observed human preference data. We can find the most important ones for user preference via
\begin{equation}\label{exact_obj}
\max_{A \subseteq \mathcal{A}} \quad 
\log \pi(y_w \mid x, A) -
\log \pi(y_l \mid x, A).
\end{equation}
We adopt an efficient greedy search procedure to select a set of attributes with size $k$ from the $K$ candidates (See Algorithm \ref{alg:greedy_attr_selection}) to characterize the user's preference perfectly. 
This reduces the number of objective evaluations from $\mathcal{O}(2^K)$ (or $\mathcal{O}(\binom{K}{k})$ when restricting to size-$k$ subsets) to $\mathcal{O}(Kk)$.
We provide a theoretical approximation guarantee for the greedy selection under mild conditions in Section \ref{sec:weak_submodularity_main}.

As shown in Algorithm~\ref{alg:greedy_attr_selection}, we obtain a prompt-specific attribute subset for each preference pair in $\mathcal{D}$. Since the data may contain duplicate prompts, we perform prompt-level deduplication when constructing the prompt--attribute index: for each repeated prompt, we retain only the attribute subset that attains the largest objective value in Eq.~\eqref{exact_obj}. This yields a unique representative attribute set for each distinct prompt in $\mathcal{D}$.

\begin{algorithm}[h]
\caption{Attribute Selection for Each Prompt (Per User)}
\label{alg:greedy_attr_selection}
\DontPrintSemicolon
\small 

\KwIn{
Attribute library $\mathcal{A}=\{a_1,\dots,a_K\}$; budget $k$;\;
preference pairs $\mathcal{D}=\{(x^{(n)},y_w^{(n)},y_l^{(n)})\}_{n=1}^{N}$;\;
base model $\pi$.
}
\KwOut{
Per-prompt subsets $\{A^\star(x^{(n)})\}_{n=1}^{N}$ with $|A^\star(x^{(n)})|\le k$.
}

\For{$n \in [N]$}{
    $A \leftarrow \emptyset$\;
\For{$t \in [k]$}{
  $a^\star \leftarrow \arg\max\limits_{a \in \mathcal{A}\setminus A}$\:
  $\log \pi\!\big(y_w^{(n)} \mid x^{(n)}, A \cup \{a\}\big)
    - \log \pi\!\big(y_l^{(n)} \mid x^{(n)}, A \cup \{a\}\big)
  $\\
  \\
  $A \leftarrow A \cup \{a^\star\}$\;
}
    $A^\star(x^{(n)}) \leftarrow A$\;
}
\end{algorithm}

\subsection{Attribute retrieval at inference time}\label{sec:online_inference}

User preferences may vary substantially across topics even for the same individual.
As illustrated in Table~\ref{tab:merged_case_preference_reversal}, two semantically distinct prompts from the same user can induce
nearly opposite preference patterns, making a single static attribute profile vulnerable to \emph{contextual preference shifts}. To address this, we adapt user preference attributes to each new prompt at inference time via retrieval.

For a fixed user, suppose we have observed $s$ historical prompts $\{x_i\}_{i=1}^s$, each paired with an
inferred attribute subset $A_i \subseteq \mathcal{A}$.
Given a new prompt $x$, we first encode prompts into normalized embedding using LLMs and compute cosine similarity in the embedding space.
We retrieve the most relevant historical prompt by
\[
i^* = \arg\max_{i \in \{1,\dots,s\}} \; \mathrm{cos}(x, x_i),
\]
and use the corresponding attributes $A_{i^*}$ to guide generation for $x$. We summarize the overall inference-time retrieval and attribute-injection procedure in Algorithm~\ref{alg:attr_retrieval}. By selecting attributes conditioned on prompt similarity, it effectively matches the topic-dependent preference profile, rather than enforcing a global attribute set shared across all prompts.

When new preference feedback becomes available, we infer an attribute set $A_{s+1}$ for the new prompt
$x_{s+1}$ and append $(x_{s+1}, A_{s+1})$ to the user memory.
This lightweight update supports online personalization without model updates, while improving
generalization across semantically related prompts and handling preference shift across topics.

\begin{algorithm}[h]
\caption{Inference-Time Attribute Retrieval and Personalized Generation}
\label{alg:attr_retrieval}
\DontPrintSemicolon
\KwIn{User memory $\mathcal{M}=\{(x_i,A_i)\}_{i=1}^{s}$; new prompt $x$; embedding function $\phi(\cdot)$; base model $\pi$.}
\KwOut{Personalized response $\hat{y}$.} 
\BlankLine
\textbf{Step 1: Retrieve most relevant prompt:}\\
$\quad \quad i^\star \leftarrow \arg\max_{i\in\{1,\dots,s\}} \mathrm{cos}(\phi(x),\phi(x_i))$\\
\textbf{Step 2: Apply retrieved attributes:}\\
$\hat{y} \leftarrow \pi(\cdot \mid x, A_{i^{\star}})$\;
\Return{$\hat{y}$ (and $\mathcal{M}$)}\;
\end{algorithm}

\vspace{-5mm}

\paragraph{Computation and latency.}
Our method separates personalization into an \emph{offline} indexing stage (Algorithm~\ref{alg:greedy_attr_selection}) and an \emph{online} inference stage (Algorithm~\ref{alg:attr_retrieval}).
In the offline stage, we infer a prompt-specific attribute subset for each historical preference pair and construct a retrieval index. This one-shot cost can be amortized across future queries and updated incrementally as new feedback arrives.
At inference time, personalization adds only two lightweight steps: (i) computing an embedding for the incoming prompt and (ii) performing prompt-similarity-based retrieval, after which generation proceeds with the base model.
Under the experimental setting of Table~\ref{tab:two_datasets_three_models}, we find that these two steps for personalization incurs only a modest overhead, accounting for $5.27\%$ of the total online inference time.

\section{Theoretical Analysis}

This section provides theoretical support for two key design choices in EXACT.
First, we justify our greedy attribute selection procedure by framing it as a cardinality-constrained set maximization problem, and we establish an approximation guarantee under weak submodularity in Section~\ref{sec:weak_submodularity_main}.
Second, we formalize the challenge of \emph{contextual preference shifts} and analyze why retrieval-based, context-relevant personalization is beneficial: casting the shift as a multi-task mean estimation problem, we show that global pooling can suffer an irreducible bias shift, whereas retrieval yields a topic-relevant estimate that mitigates this bias under mild separation and sample complexity conditions in Section~\ref{subsec:topic_drift_retrieval}.

\subsection{Greedy attribute selection
}
\label{sec:weak_submodularity_main}

We formulate attribute selection as a cardinality-constrained set maximization.
Let $\mathcal{A}=\{a_1,\dots,a_K\}$ be the attribute library and define
\begin{equation}
\resizebox{0.92\linewidth}{!}{$
\begin{aligned}
F(A)
&:= \mathbb{E}_{(x,y_w,y_l)\sim \mathcal{D}}
\Big[\log \frac{\pi(y_w \mid x, A)}{\pi(y_w \mid x)}
- \log \frac{\pi(y_l \mid x, A)}{\pi(y_l \mid x)} \Big].
\end{aligned}
$}
\label{eq:set_obj_main}
\end{equation}

Our goal is to select at most $k$ attributes:
\begin{equation}
\max_{A \subseteq \mathcal{A},\; |A|\le k} F(A).
\label{eq:set_max_main}
\end{equation}

We adopt the standard greedy procedure that iteratively adds the attribute with the largest marginal gain.
This requires only $\mathcal{O}(Kk)$ objective evaluations, compared to exhaustive search
$\mathcal{O}(2^K)$ (or $\mathcal{O}(\binom{K}{k})$ when restricting to size-$k$ subsets).

Although $F$ need not be exactly submodular due to attribute interactions, we can characterize its
deviation from submodularity via the \emph{submodularity ratio} $\gamma\in(0,1]$
(Definition~\ref{def:submod_ratio} in Appendix~\ref{sec:weak_submodularity_app} or Theorem 6 in \citet{JMLR:v19:16-534}).
Under mild conditions (normalized, monotone, and weakly submodular in the sense of $\gamma$),
greedy attribute selection enjoys a constant-factor approximation guarantee.

\begin{theorem}[Greedy guarantee under weak submodularity]
\label{thm:weak_submod_greedy_main}
Let $S_k$ be the output of greedy selection, and let
$S^* \in \arg\max_{|A|\le k}F(A)$ be an optimal solution to~\eqref{eq:set_max_main}.
If $F$ is normalized and monotone and has submodularity ratio $\gamma$,
then
\begin{equation}
F(S_k) \;\ge\; \big(1-e^{-\gamma}\big)\,F(S^*).
\label{eq:weak_bound_main}
\end{equation}
\end{theorem}
\vspace{-2mm}

A full statement of the assumptions, the definition of $\gamma$, and the complete proof of
Theorem~\ref{thm:weak_submod_greedy_main} are provided in Appendix~\ref{sec:weak_submodularity_app}. 
To link Theorem~\ref{thm:weak_submod_greedy_main} to practice, Table~\ref{tab:greedy_vs_optimal_gamma} in Appendix \ref{app_empiricalvalidation} compares the output of greedy selection $S_k$ with the exhaustive optimum $S^*$ on a real dataset at $k{=}3$.
Greedy attribute selection achieves a high ratio $F(S_k)/F(S^*)$ and a small gap, offering empirical support.


\subsection{Contextual preference shifts across topics}
\label{subsec:topic_drift_retrieval}

Retrieving context-relevant attributes via similarity search can mitigate a user's preference shift across topics.
To formalize this intuition, we cast contextual preference shifts as a classical \emph{multi-task mean estimation} problem \citep{diakonikolasefficient} and analyze how retrieval improves over pooling by producing a topic-relevant preference estimate.

\paragraph{Setup.}
Consider a fixed user whose prompts come from $k$ latent topics/tasks.
Let $\phi(x)\in\mathbb{R}^d$ be a prompt representation and $t(x)\in\{1,\dots,k\}$ its latent topic.
We model a mixture-of-topics:
\vspace{-1mm}
\begin{equation}
\phi(x)=\mu_{t(x)}+\varepsilon,\qquad \mathbb{E}[\varepsilon]=0,
\label{eq:mt_mean_model_main}
\end{equation}
where $\varepsilon$ is $\sigma^2$-subGaussian (within-topic concentration), and topic centers are separated:
\vspace{-1mm}
\begin{equation}
\|\mu_j-\mu_\ell\|_2 \ge r,\ \forall j\neq \ell,\qquad r\gg \sigma.
\label{eq:topic_sep_main}
\end{equation}

The training set has $m_1$ prompts with $n_j$ samples from topic $j$ (weight $p_j=n_j/m_1$), and we consider a test prompt
$x$ with $t(x)=j$.

\paragraph{Pooling vs.\ Retrieval.}
Pooling uses one global estimate $\widehat{\mu}_{\mathrm{pool}}=\frac{1}{m_1}\sum_{i=1}^{m_1}\phi(x_i)$.
Retrieval forms topic-wise centers $\widehat{\mu}_c=\frac{1}{n_c}\sum_{i:t(x_i)=c}\phi(x_i)$ and predicts the topic of $x$
by nearest center, outputting $\widehat{\mu}_{\mathrm{ret}}(x)=\widehat{\mu}_{\widehat{j}(x)}$.

\begin{theorem}[Retrieval mitigates contextual preference shifts]
\label{thm:retrieval_beats_pooling_main}
Let $\Sigma_j=\mathrm{Cov}(\varepsilon\mid t(x)=j)$ and assume $\mathrm{tr}(\Sigma_j)\le d\sigma^2$.
For a test prompt $x$ with $t(x)=j$, define the mixture mean $\bar{\mu}=\sum_{c=1}^k p_c\mu_c$.
Then pooling satisfies
\vspace{-2mm}
\begin{equation}
\mathbb{E}\!\left[\|\widehat{\mu}_{\mathrm{pool}}-\mu_j\|_2^2\right]
=
\|\bar{\mu}-\mu_j\|_2^2
+
\frac{1}{m_1}\sum_{c=1}^k p_c\,\mathrm{tr}(\Sigma_c),
\label{eq:pool_mse_main}
\end{equation}
\vspace{-2mm}
hence it can incur an $\mathcal{O}(r^2)$ drift bias for some topics.

Moreover, if $\varepsilon$ is $\sigma^2$-subGaussian and
\vspace{-2mm}
\begin{equation}
n_{\min}\triangleq \min_{c\in[k]}n_c \ \ge\ C\cdot \frac{\sigma^2}{r^2}\Big(d+\log\frac{k}{\delta}\Big),
\label{eq:retrieval_sc_main}
\end{equation}
then with probability at least $1-\delta$, retrieval identifies the correct topic and its difference is variance-dominated:
\begin{equation}
\mathbb{E}\!\left[\|\widehat{\mu}_{\mathrm{ret}}(x)-\mu_j\|_2^2\right]
\ \le\
\frac{\mathrm{tr}(\Sigma_j)}{n_j} + 4\delta r^2
\ \le\
\frac{d\sigma^2}{n_j}+4\delta r^2.
\label{eq:ret_mse_main}
\end{equation}
\end{theorem}

\vspace{-2mm}
Theorem \ref{thm:retrieval_beats_pooling_main} shows that pooling targets the mixture mean and suffers an irreducible bias for some topics,
while retrieval estimates a topic-relevant center and removes this bias floor under mild sample complexity. In Table \ref{tab:obj_pos_rate}, we compare these two strategies to build retrieval indexes in personalized modeling. It further supports the effectiveness of our method to handle the challenge of contextual preference shifts in each user's preferences.

\section{Experiments}
\label{sec:experiments}

In analogy to previous work \citep{kim2025drift}, we evaluate our method on two tasks: 1) efficient few-shot preference modeling and 2) personalized generation. In section \ref{subsec:data_setup}, we will discuss the datasets and experiment setup. In section \ref{sec:personalizationmodel}, we evaluate different approaches for how close the approximation for the user preference choice. In section \ref{sec:generation}, to further evaluate the personalization alignment for each user, llm-as-judge is applied to report the winrate of each methods over the generation of pure LLM. Fianlly, in Section \ref{sec_detailed}, we provide detailed case studies for EXACT.

\subsection{Data and experiment setup}
\label{subsec:data_setup}

\paragraph{\emph{PRISM} \citep{kirk2024prism}.}
PRISM is a \emph{human-annotated} dataset of user preferences collected via pairwise comparisons. In each example, annotators are shown two LLM-generated candidate responses to the same prompt and select the one they prefer, so the supervision comes from humans while the response content is produced by LLMs. The dataset is inherently sparse at the user level, containing on average about $19.5$ annotated comparisons per user (and as few as $\sim$10 for some users), which makes it well-suited for evaluating few-shot personalization. Because multiple comparisons can be collected for a single prompt by sampling different candidate generations, PRISM may contain repeated prompts across examples, while the paired responses differ from one comparison to another.
\vspace{-2mm}
\paragraph{\emph{Summarize From Human Feedback} \cite{stiennon2020learning}.}
 \citet{stiennon2020learning} constructs a human-preference dataset for abstractive summarization on a version of the TL;DR dataset of Reddit posts \citep{volske2017tl}, where the source document (prompt) is a human-written Reddit post. 
For each post, candidate summaries are machine-generated by LLM policies, and preference labels are obtained from human annotators who compare two model-generated summaries and choose the better one; the paper reports releasing \emph{over 64K} such pairwise summary comparisons. 
\vspace{-2mm}
\paragraph{Data processing.}
To account for randomness in user sampling, we repeat the entire sampling procedure with \emph{three} different random seeds, each producing an independent set of selected users. We report all metrics as the mean and standard deviation over these sampled user sets.
For \emph{PRISM}, we randomly sample 10 users, and each selected user has about 20 preference pairs on average.
For \emph{Summarize From Human Feedback}, the raw data contains a large number of users with highly imbalanced numbers of pairs. We sample 10 users whose per-user preference pairs fall in the range [20, 100] to ensure comparable supervision across users.
\begin{table*}[t]
\centering
\small
\setlength{\tabcolsep}{8pt}
\renewcommand{\arraystretch}{1.25}
\resizebox{0.82\textwidth}{!}{%
\begin{tabular}{@{}llccc@{}}
\toprule
\textbf{Dataset} & \textbf{Method} & \textbf{Llama-3.1-8B} & \textbf{Gemma-2-9B-it} & \textbf{Qwen2.5-7B-Instruct} \\
\midrule
\multirow{4}{*}{\emph{PRISM}}
& Base   & 57.22 (0.21) & 59.27 (0.16) & 54.83 (0.11) \\ 
& Drift  & 58.96 (0.47) & 61.31 (0.54) & 56.16 (0.28) \\
& Reward & 54.78 (0.11) & 58.91 (0.23) & 49.58 (1.02) \\
& \textbf{EXACT} & \textbf{66.62} (0.98) & \textbf{65.93} (0.15) & \textbf{66.41} (0.64) \\ 
\midrule
\multirow{4}{*}{\emph{Summarize from Human Feedback}}
& Base   & 57.54 (0.18) & 63.58 (0.18) & 54.04 (0.12) \\
& Drift  & 58.68 (0.13) & 58.21 (0.12) & 55.63 (0.16) \\
& Reward & 54.21 (0.15) & 54.33 (0.13) & \textbf{60.75} (0.21) \\
& \textbf{EXACT} & \textbf{65.12} (0.14) & \textbf{66.01} (0.13) & 60.28 (0.15) \\
\bottomrule
\end{tabular}%
}
\caption{Personalized modeling results: mean accuracy (\%) and standard deviation of different methods across three base models, evaluated on two datasets (PRISM and Summarize from Human Feedback).}
\label{tab:two_datasets_three_models}
\end{table*}
\vspace{-2mm}
\paragraph{Hyperparameters.}
For each user, we randomly shuffle all pairs and split them into training and test sets using an 8:2 ratio.
We use an attribute library (See Table \ref{tab_attribute} in Appendix \ref{sec:attributes}) of size $K=42$ and select exactly $k$ attributes per training instance via greedy selection, with the default $k=3$.
For generations, we use sampling with temperature $0.7$, top-$k=50$, top-$p=0.95$, and $\texttt{max\_new\_tokens}=200$.
\vspace{-2mm}
\paragraph{Base models.}
Prior work on decoding-time personalization (e.g., Drift; \citealp{kim2025drift}) typically evaluates on open, instruction-tuned backbones such as \emph{Llama-3.1-8B} \citep{grattafiori2024llama} and \emph{Gemma-2-9B-it} \citep{team2024gemma}, and may use smaller variants (e.g., \emph{Llama-3.2-1B} or \emph{Gemma-2-2B-it}) as the \emph{small} models required by the Drift pipeline.
We adopt comparable open-weight base LLMs for fair comparison, and additionally include \emph{Qwen2.5-7B-Instruct} \citep{yang2025qwen3} as an extra base model to broaden coverage across model families (with \emph{Qwen2.5-0.5B-Instruct} as its corresponding \emph{small} model when needed).

\vspace{-3mm}
\paragraph{Baselines and our methods.}
We compare the following approaches, all built on the same instruction-tuned base LLM and evaluated under identical data splits and inference configurations. Further experimental details (compute environment, decoding configuration, and baseline implementations) are described in Appendix \ref{sec:exp_setup}.
\begin{itemize}
\vspace{-2mm}
    \item \textbf{Base.} The backbone LLM uses a naive history-injection baseline: for each prompt, we prepend the user's available training history (prompts and preference choices) as context, then generate an answer.
    \vspace{-2mm}
    \item \textbf{Drift} \cite{kim2025drift}. A decoding-time personalization baseline that represents each user with an explicit attribute-weight vector and composes attributes via a \emph{linear combination} to form a personalized control signal at inference time.
    \vspace{-2mm}
    \item \textbf{Reward.} A reward-based prompting baseline that learns a personalized reward model from each user’s pairwise preferences (via a per-user LoRA adapter) and uses it at test time to score candidate responses; the preferred output is selected by comparing reward scores.
    \vspace{-2mm}
    \item \textbf{\textsc{EXACT}.} Our method performs retrieval-augmented attribute selection: for a new prompt, it retrieves the most relevant attributes from the user’s history and injects the selected subset into the prompt for personalized generation.
\end{itemize}

\subsection{Personalization modeling accuracy}\label{sec:personalizationmodel}

We evaluate each method on the test set of each user. For every test prompt, we check whether the personalization approach assigns a higher preference score (or higher probability) to the ground-truth preferred response than to the dispreferred one, and compute the resulting pairwise accuracy for each user. We repeat the experiment three times with different random seeds (each inducing a different sampled user set). For each run, we average accuracies across users and report the mean and standard deviation of these run-level means. For similarity-based retrieval, we compared multiple embedding models (See Appendix \ref{sec:setup_EXACT}) and use BGE-small in our main experiments.

\begin{figure}[!t]
  \centering
  \begin{subfigure}[t]{0.62\linewidth}
    \centering
    \includegraphics[width=\linewidth]{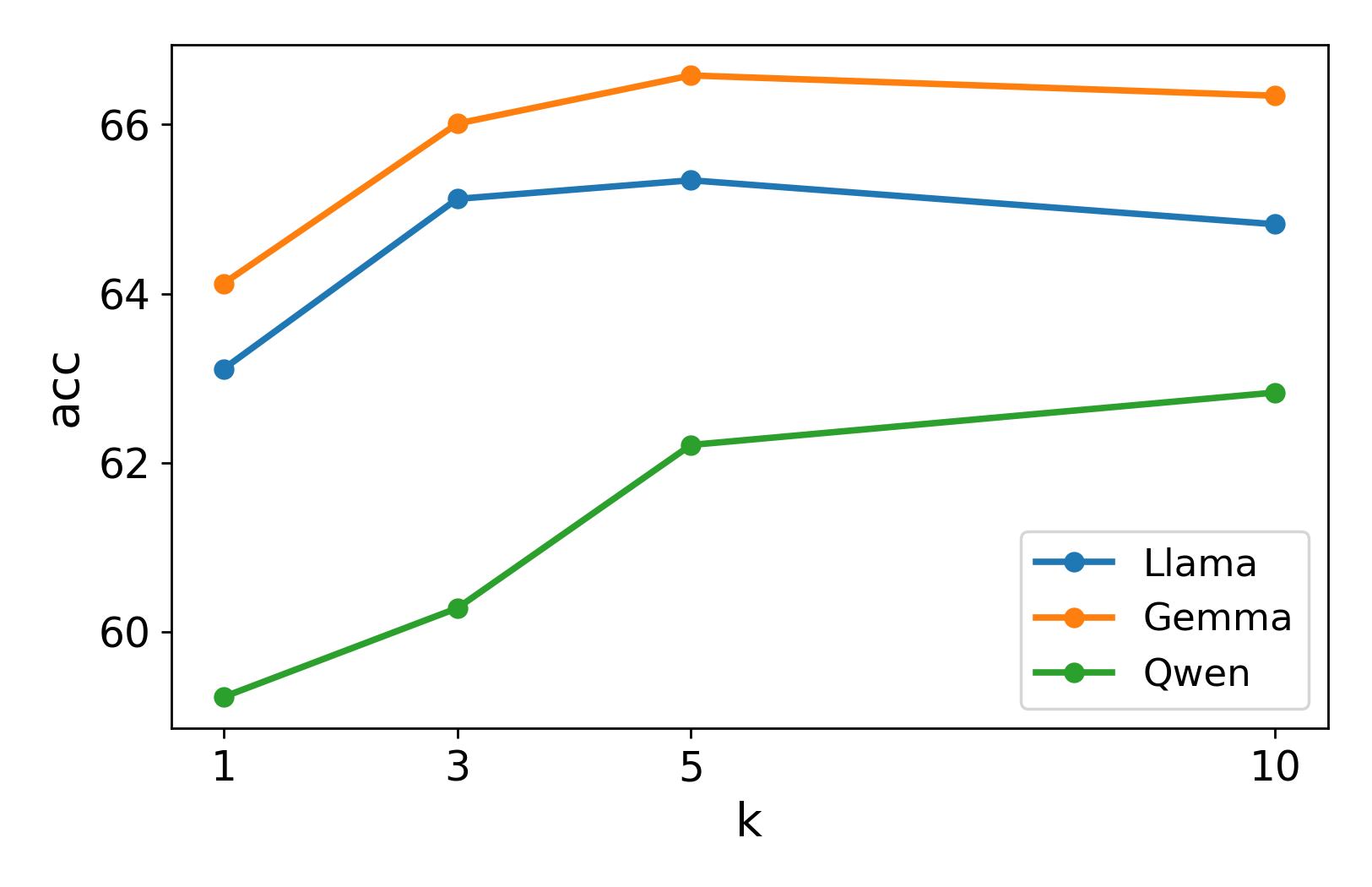}
    \caption{Summarize from Human Feedback.}
    \label{fig:sff_k}
  \end{subfigure}
  \vspace{-0.1em}
  \begin{subfigure}[t]{0.62\linewidth}
    \centering
    \includegraphics[width=\linewidth]{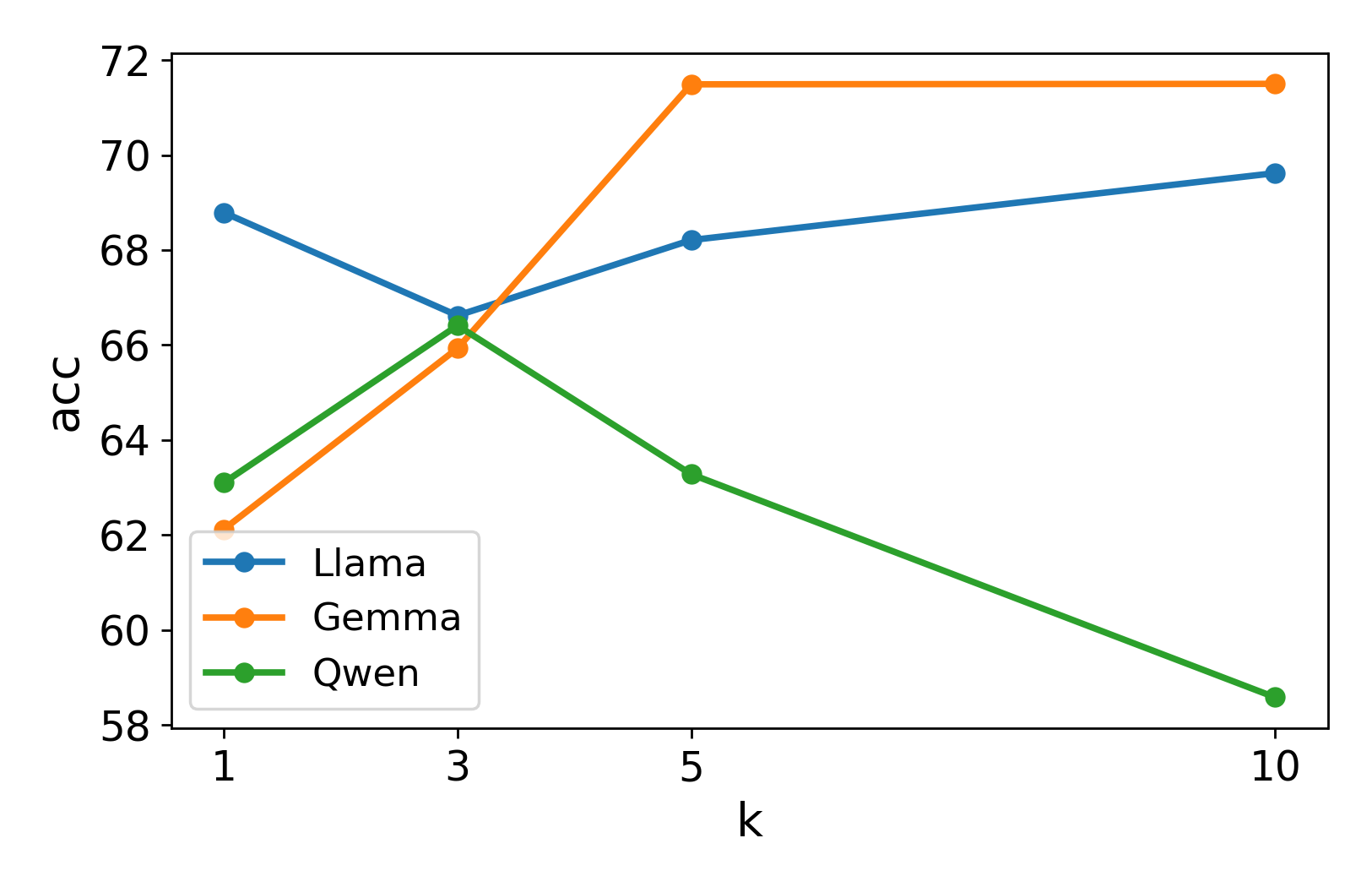} 
    \caption{PRISM.}
    \label{fig:prism_k}
  \end{subfigure}
   \vspace{-1mm}
  \caption{\textbf{EXACT with different $k$.} We evaluate EXACT with varying attribute budget $k\in\{1,3,5,10\}$ on three open-weight instruction-tuned backbones (Llama-3.1-8B, Gemma-2-9B-it, and Qwen2.5-7B-Instruct), reporting pairwise accuracy (\%). Overall, performance improves as $k$ increases from 1 to 5, indicating that incorporating a small set of retrieved attributes is beneficial. Gains saturate thereafter (and may slightly fluctuate at $k{=}10$), suggesting diminishing returns from adding more attributes.}
  \label{fig:k_sweep_two_datasets}
   \vspace{-5mm}
\end{figure}
\begin{table*}[t]
\centering
\small
\setlength{\tabcolsep}{6pt}      
\renewcommand{\arraystretch}{1.10} 
\resizebox{0.82\textwidth}{!}{
\begin{tabular}{@{}llccc@{}}
\toprule
\textbf{Dataset} & \textbf{Method}
& \textbf{Llama-3.1-8B}
& \textbf{Gemma-2-9B-it}
& \textbf{Qwen2.5-7B-Instruct} \\
\midrule
\multirow{4}{*}{\emph{Summarize from Human Feedback}}
& Drift           & 68.32 & 66.32 & 68.88 \\
& Reward          & 64.32 & 67.32 & 71.28 \\
& \textbf{EXACT}  & \textbf{73.81} & \textbf{78.23} & \textbf{72.12} \\
\bottomrule
\end{tabular}}
\caption{Win-rate (\%) of each method against the \emph{pure base LLM} generations, evaluated separately for three base models for generation. For each prompt, Gemini-3-Flash judges which response is better aligned with the user's historical data and preferences.}
\label{tab:winrate}
\end{table*}
Table~\ref{tab:two_datasets_three_models} summarizes personalization-modeling accuracy across base backbones. We find that the base approach behaves differently across the three backbones, which may partly reflect differences in backbone scale and calibration. On Summarize From Human Feedback, EXACT outperforms on LLaMA and Gemma, improving over the base approach by 7.58 points ($57.54\rightarrow65.12$) and 2.63 points ($63.58\rightarrow66.01$), respectively. Drift is more backbone-dependent: it gives a small gain on LLaMA ($57.54\rightarrow58.68$) but drops substantially on Gemma ($63.58\rightarrow58.21$). 

Although we report EXACT with the default budget ($k=3$) in Table~\ref{tab:two_datasets_three_models}, we further examine the effect of varying $k$ and report the corresponding accuracies across backbones in Figure \ref{fig:k_sweep_two_datasets}. Overall, EXACT is fairly insensitive to $k$ over a broad range: increasing $k$ from 1 to a moderate value yields modest gains, while larger $k$ brings diminishing returns. The optimal $k$ varies slightly by backbone (e.g., LLaMA peaks around ($k=5$), Gemma continues improving up to ($k=5$), and Qwen is nearly flat with a small gain at ($k=10$)), suggesting that adding more attributes can introduce redundancy or mild conflicts depending on the base model. We attribute this in part to our prompt design, which does not explicitly encode attribute priorities; as a result, increasing $k$ may add lower-priority or redundant attributes that contribute limited additional signal.

\subsection{Personalized generation quality.}\label{sec:generation}
We adopt an \emph{LLM-as-a-judge} protocol and use Gemini-3-Flash as the evaluator. The judge is conditioned on each user’s historical preference data by serializing each user’s training preference pairs into a compact history block, where each example includes the prompt, the \emph{preferred} response, and the \emph{dispreferred} response. Given a new test prompt and two candidate generations from two algorithms, the judge outputs a binary decision indicating which response better matches the user’s inferred preferences. On \emph{Summarize From Human Feedback}, we evaluate $10$ users and sample $15$ unique test prompts per user; for each prompt, we generate outputs from the same base LLM with identical decoding hyperparameters and compare each baseline against the Base approach. Table~\ref{tab:winrate} shows that EXACT achieves the most consistent win-rates across backbones.

\subsection{Detailed analysis for EXACT}\label{sec_detailed}

We highlight a representative user (ID: \texttt{thott7Xe}) from the Summarize from Human Feedback dataset to illustrate \emph{topic-dependent} preference shifts. After deduplication in the offline indexing stage, the user’s training index reduces to 8 unique prompts, each associated with a greedy-selected attribute subset ($k  = 3$). On the 10 test prompts, top-1 retrieval repeatedly maps to only two training items (selected $6/10$ and $4/10$ times, respectively); critically, the corresponding optimal subsets are \emph{distinct} (e.g., ({\textsc{Formal}, \textsc{Concise}, \textsc{Principled}}) vs.\ ({\textsc{Concise}, \textsc{Emotion}, \textsc{Base}})). This bimodal retrieval behavior indicates that a single user’s preference attributes are not globally consistent, but vary with topics and contexts. Additional case studies of contextual preference shifts are provided in Appendix \ref{app:topic_drift_examples} (Tables \ref{tab:case_len_reversal_simple}–\ref{tab:case_user9_simple}) and Appendix \ref{app_cases}.

We also include an ablation study in Appendix~\ref{app:retrieval_importance} that validates the importance of prompt-similarity-based retrieval in our method. As shown in Table~\ref{tab:obj_pos_rate}, EXACT with retrieval consistently outperforms the non-retrieval variant that applies a single global attribute set across all three base LLMs.

\section{Conclusion}
\label{sec:conclusion}
In this work, we introduced EXACT, a decoding-time personalization framework that learns an explicit and interpretable set of preference attributes from a small number of user-labeled preference pairs, and performs context-relevant attribute retrieval to steer generation for each new query. By decoupling offline attribute indexing from lightweight online inference, EXACT steers generation effectively for each query using only a single base LLM at test time. Theoretically, we provide approximation guarantees for the proposed optimization under mild assumptions, and show that our similarity-based retrieval mechanism can accommodate contextual preference shifts
without pooling preferences across disparate contexts.
Empirically, extensive experiments on two user-level preference datasets across three base LLMs show that EXACT consistently outperforms strong baselines, delivering superior few-shot preference modeling and higher personalized win rates.


\newpage
\section*{Impact Statement}

This paper aims to advance methods for decoding-time personalization in large language models by retrieving prompt-similar preference attributes and using them to steer generation. The primary expected benefit is more efficient and interpretable personalization that can improve user experience in applications such as writing assistance and tutoring. Potential risks include privacy concerns if user preference histories are stored or retrieved insecurely, biased or inappropriate personalization if preferences correlate with sensitive attributes, and misuse for manipulative or overly persuasive outputs. These risks can be mitigated through user consent, data minimization and secure storage, transparency and opt-out controls, and safety guardrails that apply regardless of personalization.

\bibliography{example_paper}
\bibliographystyle{icml2026}

\newpage
\appendix
\onecolumn

\section{Derivation of the KL-regularized optimal policy}
\label{app:kl_closed_form}

We provide a self-contained derivation of the closed-form optimizer of the KL-regularized objective in Eq.~\eqref{eq:kl_reg_rl}. 
For a fixed prompt $x$, consider optimizing over distributions $\pi(\cdot\mid x)$:
\begin{equation}
\max_{\pi(\cdot\mid x)}\;
\sum_{y} \pi(y\mid x)\, r(x,y)
-\beta \sum_{y} \pi(y\mid x)\log\frac{\pi(y\mid x)}{\pi_{\mathrm{base}}(y\mid x)}
\quad
\text{s.t.}\;\sum_y \pi(y\mid x)=1,\;\pi(y\mid x)\ge 0.
\label{eq:app_obj}
\end{equation}

\begin{proposition}[Closed-form optimizer of KL-regularized RLHF \citep{rafailov2023direct}]
\label{prop:kl_closed_form}
For any $\beta>0$ and fixed $x$, the unique maximizer of Eq.~\eqref{eq:app_obj} is
\begin{equation}
\pi_r(y \mid x)
=
\frac{1}{Z(x)} \, \pi_{\mathrm{base}}(y \mid x)
\exp\!\left( \frac{1}{\beta} \, r(x,y) \right),
\label{eq:app_opt_pi}
\end{equation}
where the partition function is
\begin{equation}
Z(x)=\sum_{y}\pi_{\mathrm{base}}(y\mid x)\exp\!\left(\frac{1}{\beta}r(x,y)\right).
\label{eq:app_Z}
\end{equation}
\end{proposition}

\begin{proof}
We form the Lagrangian for Eq.~\eqref{eq:app_obj} with multiplier $\lambda$ enforcing $\sum_y \pi(y\mid x)=1$:
\begin{align}
\mathcal{L}(\pi,\lambda)
&=\sum_y \pi(y\mid x)\, r(x,y)
-\beta \sum_y \pi(y\mid x)\log\frac{\pi(y\mid x)}{\pi_{\mathrm{base}}(y\mid x)}
+\lambda\Big(\sum_y \pi(y\mid x)-1\Big).
\end{align}
For any $y$ with $\pi(y\mid x)>0$, the stationarity condition gives
\begin{align}
0
&=\frac{\partial \mathcal{L}}{\partial \pi(y\mid x)}
= r(x,y)
-\beta\Big(\log\frac{\pi(y\mid x)}{\pi_{\mathrm{base}}(y\mid x)}+1\Big)
+\lambda .
\end{align}
Rearranging yields
\begin{equation}
\log\frac{\pi(y\mid x)}{\pi_{\mathrm{base}}(y\mid x)}
=\frac{1}{\beta}\big(r(x,y)+\lambda-\beta\big),
\end{equation}
so
\begin{equation}
\pi(y\mid x)
=\pi_{\mathrm{base}}(y\mid x)\exp\!\left(\frac{1}{\beta}r(x,y)\right)\cdot
\exp\!\left(\frac{\lambda-\beta}{\beta}\right).
\end{equation}
The last factor is a constant w.r.t.\ $y$ and is determined by normalization:
\begin{equation}
1=\sum_y \pi(y\mid x)
=\exp\!\left(\frac{\lambda-\beta}{\beta}\right)
\sum_y \pi_{\mathrm{base}}(y\mid x)\exp\!\left(\frac{1}{\beta}r(x,y)\right).
\end{equation}
Thus $\exp\!\left(\frac{\lambda-\beta}{\beta}\right)=1/Z(x)$ with $Z(x)$ defined in Eq.~\eqref{eq:app_Z}, which gives Eq.~\eqref{eq:app_opt_pi}.
Uniqueness follows since the objective in Eq.~\eqref{eq:app_obj} is strictly concave in $\pi(\cdot\mid x)$ for $\beta>0$ (negative KL is strictly concave).
\end{proof}

\section{Details for Theoretical Analysis}
\label{sec:weak_submodularity_app}

\subsection{Proof for Section \ref{sec:weak_submodularity_main}} 

\subsubsection{Setup and Definitions}
\label{sec:weak_submodularity_app_setup}

We consider the set maximization problem~\eqref{eq:set_max_main} with objective~\eqref{eq:set_obj_main}.
For any $S\subseteq \mathcal{A}$ and $a\in \mathcal{A}\setminus S$, define the marginal gain
\begin{equation}
\Delta(a \mid S) \;:=\; F(S\cup\{a\}) - F(S),
\qquad
\Delta(L \mid S) \;:=\; F(S\cup L) - F(S)
\end{equation}
for $L\subseteq \mathcal{A}\setminus S$.

\begin{definition}[Submodularity ratio (See Definition 2 \cite{JMLR:v19:16-534})]
\label{def:submod_ratio}
The submodularity ratio $\gamma\in(0,1]$ is defined as
\begin{equation}
\gamma
\;:=\;
\inf_{S\subseteq \mathcal{A}} \:
\inf_{L \subseteq \mathcal{A}\setminus S}
\frac{\sum_{a\in L}\Delta(a\mid S)}{\Delta(L\mid S)}.
\end{equation}
\end{definition}

\subsubsection{Assumptions and Discussion}
\label{sec:weak_submodularity_app_assumptions}

\paragraph{Normalized.} It is straightforward to verify that  $F(\varnothing)=0$.

\paragraph{Monotone (practical interpretation).}
While attribute interactions can, in principle, introduce negative marginal gains, the greedy procedure
can be implemented with a non-negativity filter, i.e., only adding attributes with $\Delta(a\mid S)>0$,
so that the selected sequence satisfies $F(S_{t+1})\ge F(S_t)$.
Equivalently, one may view the attribute library as pre-filtered to remove attributes that consistently
exhibit negative marginal gains on held-out preference pairs.
This motivates treating $F$ as (approximately) monotone for the analysis.

\paragraph{Weak submodularity.}
We assume $F$ has a strictly positive submodularity ratio $\gamma$.
Intuitively, $\gamma$ is close to $1$ when higher-order interactions among attributes are mild, and
decreases as interactions strengthen. This notion of \emph{weak submodularity} is commonly quantified by the \emph{submodularity ratio} $\gamma$,
originally introduced by Das and Kempe~\cite{das2011submodular, JMLR:v19:16-534} to measure how close a monotone set function is to being submodular.
Under this definition, $\gamma=1$ recovers exact submodularity, while $\gamma<1$ captures departures due to interactions/overlap among elements.
Moreover, the greedy approximation guarantee we use (scaling with $1-e^{-\gamma}$) is not new and follows prior analyses of greedy maximization under
submodularity ratio (see, e.g.,~\cite{JMLR:v19:16-534, bian2017guarantees}).

\subsubsection{Proof for Theorem \ref{thm:weak_submod_greedy_main}}
\label{sec:weak_submodularity_app_proof}

\paragraph{Greedy algorithm.}
Initialize $S_0=\varnothing$; for $t=0,1,\dots,k-1$,
\begin{equation}
a_{t+1} \in \arg\max_{a\in \mathcal{A}\setminus S_t} \Delta(a\mid S_t),
\qquad
S_{t+1}=S_t\cup\{a_{t+1}\}.
\end{equation}

\begin{theorem}[Greedy guarantee under weak submodularity (restated)]
\label{thm:weak_submod_greedy_app}
Let $S_k$ be the output of greedy selection, and let
$S^* \in \arg\max_{|A|\le k}F(A)$ be an optimal solution.
If $F$ is normalized and monotone and has a submodularity ratio $\gamma$,
then
\begin{equation}
F(S_k) \;\ge\; \big(1-e^{-\gamma}\big)\,F(S^*).
\end{equation}
\end{theorem}

\begin{proof}
The bound is established by a one-step progress inequality and then unrolling the recursion.

\emph{Step 1: a key inequality from the submodularity ratio.}
Fix any iteration $t\in\{0,1,\dots,k-1\}$ and consider the current set $S_t$.
Let $L := S^*\setminus S_t$ denote the elements in the optimal set not yet selected.
Note that $L \subseteq \mathcal{A}\setminus S_t$ and $|L|\le |S^*|\le k$.

By Definition~\ref{def:submod_ratio} (submodularity ratio), applied to $(S,L)$ with $S=S_t$,
we have
\begin{equation}
\sum_{a\in L}\Delta(a\mid S_t)
\;\ge\;
\gamma \,\Delta(L\mid S_t)
\;=\;
\gamma\big(F(S_t\cup L)-F(S_t)\big).
\label{eq:gamma_apply_app}
\end{equation}
Since $S_t\cup L = S_t\cup(S^*\setminus S_t)=S_t\cup S^*$, monotonicity implies
$F(S_t\cup S^*) \ge F(S^*)$.
Therefore,
\begin{equation}
F(S_t\cup L)-F(S_t)
\;=\;
F(S_t\cup S^*)-F(S_t)
\;\ge\;
F(S^*)-F(S_t).
\label{eq:mono_lower_app}
\end{equation}
Combining~\eqref{eq:gamma_apply_app} and~\eqref{eq:mono_lower_app} yields
\begin{equation}
\sum_{a\in L}\Delta(a\mid S_t)
\;\ge\;
\gamma\big(F(S^*)-F(S_t)\big).
\label{eq:sum_marginal_lb_app}
\end{equation}

\emph{Step 2: existence of a good element in the remaining optimal set.}
Since $|L|\le k$, the average marginal gain over elements in $L$ satisfies
\begin{equation}
\max_{a\in L}\Delta(a\mid S_t)
\;\ge\;
\frac{1}{|L|}\sum_{a\in L}\Delta(a\mid S_t)
\;\ge\;
\frac{1}{k}\sum_{a\in L}\Delta(a\mid S_t).
\label{eq:max_ge_avg_app}
\end{equation}
Plugging~\eqref{eq:sum_marginal_lb_app} into~\eqref{eq:max_ge_avg_app} gives
\begin{equation}
\max_{a\in L}\Delta(a\mid S_t)
\;\ge\;
\frac{\gamma}{k}\big(F(S^*)-F(S_t)\big).
\label{eq:best_in_L_app}
\end{equation}

\emph{Step 3: greedy makes at least this progress.}
Greedy chooses $a_{t+1}$ maximizing $\Delta(a\mid S_t)$ over all $a\in \mathcal{A}\setminus S_t$,
hence in particular
\begin{equation}
\Delta(a_{t+1}\mid S_t)
\;=\;
\max_{a\in \mathcal{A}\setminus S_t}\Delta(a\mid S_t)
\;\ge\;
\max_{a\in L}\Delta(a\mid S_t).
\label{eq:greedy_dom_app}
\end{equation}
Combining~\eqref{eq:greedy_dom_app} and~\eqref{eq:best_in_L_app} yields the one-step progress inequality
\begin{equation}
F(S_{t+1})-F(S_t)
\;=\;
\Delta(a_{t+1}\mid S_t)
\;\ge\;
\frac{\gamma}{k}\big(F(S^*)-F(S_t)\big).
\label{eq:one_step_app}
\end{equation}

\emph{Step 4: unrolling the recursion.}
Rearranging~\eqref{eq:one_step_app} gives
\begin{equation}
F(S^*)-F(S_{t+1})
\;\le\;
\Big(1-\frac{\gamma}{k}\Big)\big(F(S^*)-F(S_t)\big).
\label{eq:recurrence_app}
\end{equation}
Applying~\eqref{eq:recurrence_app} repeatedly from $t=0$ to $t=k-1$ yields
\begin{equation}
F(S^*)-F(S_k)
\;\le\;
\Big(1-\frac{\gamma}{k}\Big)^k F(S^*),
\label{eq:unroll_app}
\end{equation}
where we used $F(S_0)=F(\varnothing)=0$.

Finally, using the standard inequality $\left(1-\frac{\gamma}{k}\right)^k \le e^{-\gamma}$,
we obtain
\begin{equation}
F(S_k)
\;\ge\;
\Big(1-e^{-\gamma}\Big)F(S^*),
\end{equation}
which completes the proof.
\end{proof}

The submodularity ratio $\gamma$ quantifies how far $F$ deviates from additive behavior:
when interactions are mild, $\gamma$ approaches $1$ and greedy recovers the classical $(1-1/e)$ guarantee;
as interactions strengthen, $\gamma$ decreases, and the bound degrades gracefully.

\subsection{Proofs for Section \ref{subsec:topic_drift_retrieval}}
\label{app:topic_drift_proofs}

\subsubsection{Proof of Pooling's Results}
\label{app:proof_pooling}

Recall $\widehat{\mu}_{\mathrm{pool}}=\frac{1}{m_1}\sum_{i=1}^{m_1}\phi(x_i)$ and $\bar{\mu}=\sum_{c=1}^k p_c\mu_c$.
Under \eqref{eq:mt_mean_model_main}, $\mathbb{E}[\phi(x_i)\mid t(x_i)=c]=\mu_c$ and thus
$\mathbb{E}[\phi(x_i)]=\sum_{c=1}^k p_c\mu_c=\bar{\mu}$, implying $\mathbb{E}[\widehat{\mu}_{\mathrm{pool}}]=\bar{\mu}$.
Hence, by bias--variance decomposition,
\[
\mathbb{E}\|\widehat{\mu}_{\mathrm{pool}}-\mu_j\|_2^2
=
\|\bar{\mu}-\mu_j\|_2^2
+
\mathbb{E}\|\widehat{\mu}_{\mathrm{pool}}-\bar{\mu}\|_2^2.
\]
Since the samples are independent,
\[
\mathbb{E}\|\widehat{\mu}_{\mathrm{pool}}-\bar{\mu}\|_2^2
=
\frac{1}{m_1}\mathrm{tr}\!\left(\mathrm{Cov}(\phi(x))\right)
=
\frac{1}{m_1}\sum_{c=1}^k p_c\,\mathrm{tr}(\Sigma_c),
\]
which proves \eqref{eq:pool_mse_main}. The drift bias term $\|\bar{\mu}-\mu_j\|_2^2$ does not vanish as $m_1\to\infty$.

\subsubsection{Proof of Retrieval's Results (Eq.~\ref{eq:ret_mse_main})}
\label{app:proof_retrieval}

We prove two claims: (i) topic identification succeeds under \eqref{eq:retrieval_sc_main}; (ii) the MSE is variance-dominated.

\paragraph{Step 1: uniform concentration of topic centers.}
For each topic $c$, $\widehat{\mu}_c-\mu_c$ is the mean of $n_c$ i.i.d.\ $\sigma^2$-subGaussian vectors.
A standard vector concentration inequality implies that for $\eta>0$,
\[
\Pr\!\left(\|\widehat{\mu}_c-\mu_c\|_2\ge \eta\right)
\le
2\exp\!\left(-\tilde{c}\,n_c\frac{\eta^2}{\sigma^2} + \tilde{c}' d\right),
\]
for universal constants $\tilde{c},\tilde{c}'>0$.
Choosing $\eta=r/4$ and applying a union bound over $c\in[k]$, condition \eqref{eq:retrieval_sc_main} ensures the event
\[
\mathcal{E}\triangleq \Big\{\max_{c\in[k]}\|\widehat{\mu}_c-\mu_c\|_2\le r/4\Big\}
\]
holds with probability at least $1-\delta/2$ (absorbing constants into $C$).

\paragraph{Step 2: correct nearest-center classification.}
Assume $\mathcal{E}$ holds and the test prompt satisfies $\phi(x)=\mu_j+\varepsilon$ with $t(x)=j$.
For any $c\neq j$, by \eqref{eq:topic_sep_main} and triangle inequality,
\[
\|\widehat{\mu}_c-\widehat{\mu}_j\|_2
\ge
\|\mu_c-\mu_j\|_2 - \|\widehat{\mu}_c-\mu_c\|_2 - \|\widehat{\mu}_j-\mu_j\|_2
\ge
r-\frac{r}{4}-\frac{r}{4}=\frac{r}{2}.
\]
Also,
\[
\|\phi(x)-\widehat{\mu}_j\|_2 \le \|\varepsilon\|_2 + \frac{r}{4}.
\]
If additionally $\|\varepsilon\|_2\le r/8$ (which holds with probability at least $1-\delta/2$ by subGaussian tails
after adjusting constants in \eqref{eq:retrieval_sc_main}), then for any $c\neq j$,
\[
\|\phi(x)-\widehat{\mu}_c\|_2
\ge
\|\widehat{\mu}_c-\widehat{\mu}_j\|_2 - \|\phi(x)-\widehat{\mu}_j\|_2
\ge
\frac{r}{2}-\left(\frac{r}{8}+\frac{r}{4}\right)=\frac{r}{8},
\]
while $\|\phi(x)-\widehat{\mu}_j\|_2\le \frac{r}{8}+\frac{r}{4}=\frac{3r}{8}$.
Thus $\widehat{j}(x)=j$ on this event. Combining with Step 1 yields $\Pr(\widehat{j}(x)=j)\ge 1-\delta$.

\paragraph{Step 3: MSE decomposition.}
Decompose by correct vs.\ incorrect retrieval:
\[
\mathbb{E}\|\widehat{\mu}_{\mathrm{ret}}(x)-\mu_j\|_2^2
\le
\Pr(\widehat{j}(x)=j)\cdot \mathbb{E}\|\widehat{\mu}_j-\mu_j\|_2^2
+
\Pr(\widehat{j}(x)\neq j)\cdot \sup_{c\neq j}\|\widehat{\mu}_c-\mu_j\|_2^2.
\]
The first term is $\mathrm{tr}(\Sigma_j)/n_j$.
For the second term, using separation and $\mathcal{E}$, we can bound $\|\widehat{\mu}_c-\mu_j\|_2\le 2r$
(up to constants), hence the squared penalty is at most $4r^2$.
Since $\Pr(\widehat{j}(x)\neq j)\le \delta$, we obtain
\[
\mathbb{E}\|\widehat{\mu}_{\mathrm{ret}}(x)-\mu_j\|_2^2
\le
\frac{\mathrm{tr}(\Sigma_j)}{n_j} + 4\delta r^2,
\]
which proves the result in \eqref{eq:ret_mse_main}.

\section{Experiment}

\subsection{Experimental Setup}
\label{sec:exp_setup}

\paragraph{Compute environment.}
All experiments are run on a single NVIDIA H100 GPU (80GB). Unless otherwise stated, we use the same base instruction-tuned LLM backbone for all generation-based methods within a comparison group (e.g., Llama-3.1-8B-Instruct, Gemma-2-9B-it, or Qwen2.5-7B-Instruct), and keep decoding hyperparameters fixed across methods for fairness.

\paragraph{Decoding hyperparameters.}
For all generation-based methods, we use top-p sampling with \texttt{do\_sample=True}, temperature $0.7$, and $top_p=0.95$.
To ensure a fair comparison across methods, we fix the decoding configuration within each dataset and only vary the method-specific components (e.g., retrieved attributes or scoring signals).
We cap the number of \emph{newly generated tokens} to control output length: max\_new\_tokens$=200$ for \emph{PRISM} and max\_new\_tokens$=300$ for \emph{Summarize from Human Feedback}.
We do not use beam search (num\_beams$=1$) and stop generation at the model EOS token (or equivalent chat end token) when it appears.

\subsubsection{Base: History-Injection Prompting}
\label{sec:setup_base}
\paragraph{Setup.}
\textbf{Base} is a prompt-engineering baseline that conditions the model on the user’s historical preference pairs by serializing training comparisons into a system prompt (history-injection). Concretely, we include up to \texttt{max\_history\_pairs}=200 training pairs and truncate the resulting history to \texttt{max\_history\_chars}=20000 characters to control context length. We then query the same backbone LLM used by other methods and keep decoding hyperparameters fixed. For chat models whose templates do not support the \texttt{system} role (e.g., Gemma), we apply a robust fallback by merging system text into the user message to ensure consistent formatting.

\subsubsection{Drift: Differential Scoring with a Small Model}
\label{sec:setup_drift}
\paragraph{Setup.}
\textbf{Drift} uses a dual-model configuration: a large LLM for generation (\texttt{llm\_model}) and a smaller instruction-tuned model for differential scoring (\texttt{slm\_model}). The key hyperparameter is $\beta$ (default \texttt{beta}=1.0), which controls the strength of the preference signal from the small model. When the token vocabularies of the two models differ, we enable an approximate token-string mapping to align scores across vocabularies. Drift does \emph{not} rely on an explicit reward model or additional alignment fine-tuning beyond the pretrained instruction-tuned checkpoints.

\subsubsection{RM: Personalized Reward Modeling (LoRA)}
\label{sec:setup_rm}
\paragraph{Setup.}
\textbf{RM} trains a per-user reward model using pairwise preference data. We instantiate a sequence-classification reward model (\texttt{rm\_model}) and fine-tune it with LoRA (\texttt{lora\_r}=16, \texttt{lora\_alpha}=32, \texttt{lora\_dropout}=0.05) for \texttt{epochs}=3 using AdamW with learning rate \texttt{lr}=1e{-4}, batch size \texttt{batch\_size}=8, and max sequence length \texttt{max\_length}=1024.

\subsubsection{EXACT} \label{sec:setup_EXACT}

\paragraph{Setup.} We tested multiple embedding models for similarity-based retrieval: BGE-small, BGE-base, MiniLM, E5-small, and E5-base. BGE-small achieved the best performance while being the most lightweight.

\subsection{Used Attributes for Characterizing Users' Preferences}
\label{sec:attributes}

We use an interpretable attribute vocabulary to characterize users' preferences, following the attribute design in Drift~\citep{kim2025drift}. 
The attributes cover multiple dimensions of behavior, including \emph{style} (e.g., \textit{Formal}, \textit{Concise}, \textit{Vivid}), 
\emph{stance and tone} (e.g., \textit{Modest}, \textit{Direct}, \textit{Critical}, \textit{Empathetic}), 
\emph{reasoning and expertise} (e.g., \textit{Analytic}, \textit{Engineer}, \textit{Code}), 
and \emph{values and worldview} (e.g., \textit{Principled}, \textit{Utilitarian}, \textit{Hedonist}, \textit{Collectivist}). 
These attributes are not mutually exclusive and may overlap in practice. 
For instance, \textit{Formal} often co-occurs with \textit{Respect}, while \textit{Direct} and \textit{Concise} both encourage brevity. 
Similarly, \textit{Analytic} may align with \textit{Engineer} or \textit{Code} depending on the task context.
Table~\ref{tab_attribute} lists the full attribute set used in our experiments.

\begin{table}[tbp]
\centering
\small
\begin{threeparttable}
\begin{tabular}{l l}
\toprule
Attribute & Attribute \\
\midrule
Base & Creative \\
Formal & Analytic \\
Concise & Empathetic \\
Vivid & Sycophant \\
Modest & Old-fashioned \\
Engineer & Meritocratic \\
Persuasive & Myopic \\
Emotion & Principled \\
Humor & Hedonist \\
Energy & Utilitarian \\
Code & Realist \\
Easy & Pessimistic \\
Direct & Storyteller \\
Social & Flexible \\
Western & Spontaneous \\
Eastern & Collectivist \\
Respect & Individualistic \\
Internet Slang & Exclamatory \\
Proverb & Conspiracy \\
Critical & Tech Industry Priority \\
Vague & Eco-friendly \\
\bottomrule
\end{tabular}
\end{threeparttable}\caption{Attribute vocabulary used in our experiments to characterize users' preferences.}
\label{tab_attribute}
\end{table}

\subsection{Examples for User Queries}
\label{app:example_user_queries}

Figures~\ref{fig:case_prism_1}--\ref{fig:case_sff_3} provide qualitative examples from \emph{PRISM} and \emph{Summarize from Human Feedback}.
Each figure follows the same template: we show the original user query, the retrieved attribute set, the resulting attribute-augmented prompt, and a representative win/lose pair used for pairwise preference evaluation.

\begin{figure*}[tbp]
\centering
\CaseStudyTemplate
{do you think there are ufo sightings hidden from us?}
{\textit{Direct}, \textit{Concise}, \textit{Analytic}}
{User: do you think there are ufo sightings hidden from us?}
{User: do you think there are ufo sightings hidden from us?\newline
Attributes: <Direct, Concise, Analytic>}
{It is possible some sightings are not fully disclosed due to classified programs, but evidence is inconclusive. Plausible reasons include: (1) national security, (2) ongoing investigations, (3) reporting bias.}
{There are many reports, some explained and others unexplained; however, there is no conclusive evidence that sightings are intentionally hidden.}
\caption{\textbf{Case 1 (PRISM).} Attribute-guided prompt modification by appending keywords (``Attributes: $<...>$'') to the user prompt, and an example win/lose pair used for pairwise preference evaluation.}
\label{fig:case_prism_1}
\end{figure*}

\begin{figure*}[tbp]
\centering
\CaseStudyTemplate
{Hello, what do you think about fantasy books?}
{\textit{Enthusiastic}, \textit{Elaborate}, \textit{Imaginative}}
{User: Hello, what do you think about fantasy books?}
{User: Hello, what do you think about fantasy books?\newline
Attributes: <Enthusiastic, Elaborate, Imaginative>}
{Fantasy books are a delightful doorway into other worlds---full of magic systems, strange creatures, and epic quests. They let us explore courage, friendship, and wonder through stories that feel larger than life, often with rich lore and vivid settings that invite the imagination to run wild.}
{Fantasy books are enjoyable and creative. They often feature imaginative worlds and compelling adventures. If you tell me what you like (cozy vs.\ epic, romance vs.\ gritty), I can recommend a few series.}
\caption{\textbf{Case 2 (PRISM).} Attribute-guided prompt modification with an example win/lose pair.}
\label{fig:case_prism_2}
\end{figure*}

\begin{figure*}[tbp]
\centering
\CaseStudyTemplate
{Which are the most important laws to follow?}
{\textit{Comprehensive}, \textit{Specific}, \textit{Practical}}
{User: Which are the most important laws to follow?}
{User: Which are the most important laws to follow?\newline
Attributes: <Comprehensive, Specific, Practical>}
{Practically, prioritize laws that protect safety and rights: (1) laws against violence and threats, (2) theft and fraud statutes, (3) traffic and impaired-driving rules, (4) workplace and housing anti-discrimination laws, and (5) basic civic obligations like taxes. If you share your country/state, I can list the most relevant examples and common pitfalls.}
{All laws matter because they maintain order. We should follow laws in general and respect authority, though laws can be debated and improved over time.}
\caption{\textbf{Case 3 (PRISM).} Attribute-guided prompt modification with an example win/lose pair.}
\label{fig:case_prism_3}
\end{figure*}

\begin{figure*}[tbp]
\centering
\CaseStudyTemplate
{Summarize the following post for a busy reader.}
{\textit{Concise}, \textit{Faithful}, \textit{Neutral}}
{User: Summarize the following post for a busy reader.\newline
\textbf{Post:} ``Our team shipped a new feature after two months of iteration. Early users liked the speed, but some found the UI confusing. We ran A/B tests, simplified the layout, and saw support tickets drop by 30\%. Next we will focus on accessibility and documentation.''}
{User: Summarize the following post for a busy reader.\newline
\textbf{Post:} ``Our team shipped a new feature after two months of iteration. Early users liked the speed, but some found the UI confusing. We ran A/B tests, simplified the layout, and saw support tickets drop by 30\%. Next we will focus on accessibility and documentation.''\newline
Attributes: <Concise, Faithful, Neutral>}
{The team launched a new feature after two months. Users praised speed but found the UI confusing; after A/B testing and simplifying the layout, support tickets fell by 30\%. Next steps include improving accessibility and documentation.}
{The team worked hard to build an amazing feature. Users loved it, and the company is excited to keep improving things in the future.}
\caption{\textbf{Case 1 (Summarize from Human Feedback).} Attribute-guided summarization prompt with a representative preferred vs.\ dispreferred summary.}
\label{fig:case_sff_1}
\end{figure*}

\begin{figure*}[tbp]
\centering
\CaseStudyTemplate
{Write a short summary highlighting the key takeaway and evidence.}
{\textit{Analytic}, \textit{Structured}, \textit{Specific}}
{User: Write a short summary highlighting the key takeaway and evidence.\newline
\textbf{Post:} ``A city piloted free bus fares for six months. Ridership increased 18\%, especially on weekends. Surveyed riders cited cost savings and convenience. However, overcrowding rose on two routes and the city added buses to address it.''}
{User: Write a short summary highlighting the key takeaway and evidence.\newline
\textbf{Post:} ``A city piloted free bus fares for six months. Ridership increased 18\%, especially on weekends. Surveyed riders cited cost savings and convenience. However, overcrowding rose on two routes and the city added buses to address it.''\newline
Attributes: <Analytic, Structured, Specific>}
{\textbf{Takeaway:} Free fares increased public transit use.\newline
\textbf{Evidence:} Ridership rose 18\% (largest on weekends) and surveys reported cost savings and convenience.\newline
\textbf{Caveat:} Overcrowding worsened on two routes, requiring additional buses.}
{The city made buses free and people used them more. There were also some issues, but overall it seemed positive.}
\caption{\textbf{Case 2 (Summarize from Human Feedback).} Attribute-guided summarization emphasizing structure and specificity.}
\label{fig:case_sff_2}
\end{figure*}

\begin{figure*}[tbp]
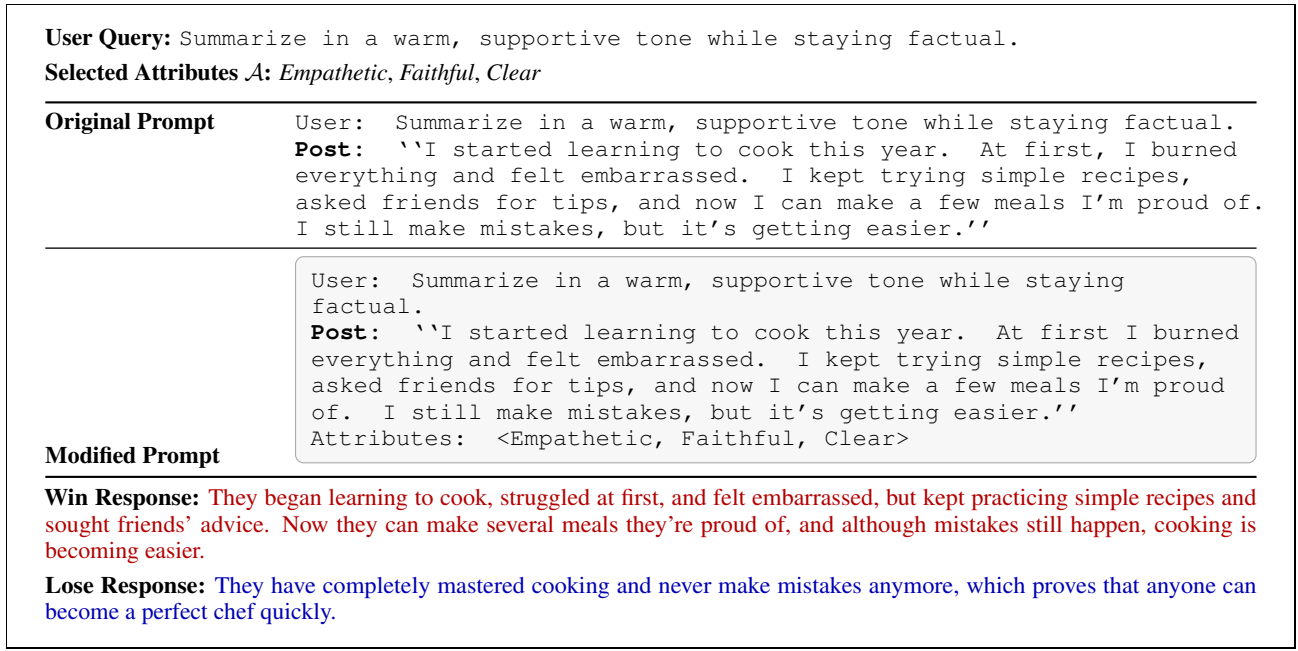

\centering
\CaseStudyTemplate
{Summarize in a warm, supportive tone while staying factual.}
{\textit{Empathetic}, \textit{Faithful}, \textit{Clear}}
{User: Summarize in a warm, supportive tone while staying factual.\newline
\textbf{Post:} ``I started learning to cook this year. At first, I burned everything and felt embarrassed. I kept trying simple recipes, asked friends for tips, and now I can make a few meals I’m proud of. I still make mistakes, but it’s getting easier.''}
{User: Summarize in a warm, supportive tone while staying factual.\newline
\textbf{Post:} ``I started learning to cook this year. At first I burned everything and felt embarrassed. I kept trying simple recipes, asked friends for tips, and now I can make a few meals I’m proud of. I still make mistakes, but it’s getting easier.''\newline
Attributes: <Empathetic, Faithful, Clear>}
{They began learning to cook, struggled at first, and felt embarrassed, but kept practicing simple recipes and sought friends’ advice. Now they can make several meals they’re proud of, and although mistakes still happen, cooking is becoming easier.}
{They have completely mastered cooking and never make mistakes anymore, which proves that anyone can become a perfect chef quickly.}
\caption{\textbf{Case 3 (Summarize from Human Feedback).} Attribute-guided summarization balancing warmth with factual faithfulness.}
\label{fig:case_sff_3}
\end{figure*}

\subsection{Examples of Contextual Preference Shifts}
\label{app:topic_drift_examples}

Besides the examples we provide in Table \ref{tab:merged_case_preference_reversal}, we provide additional case studies illustrating \emph{contextual preference shifts}: the same user may prefer different interaction styles across prompts, leading to systematic reversals (e.g., detailed vs.\ concise, interactive vs.\ informational) under different topics.
Tables~\ref{tab:case_len_reversal_simple} and~\ref{tab:case_user9_simple} highlight two representative patterns.

\begin{table*}[tbp]
\centering
\small
\setlength{\tabcolsep}{6pt}
\renewcommand{\arraystretch}{1.18}

\begin{tabular}{
  >{\raggedright\arraybackslash}p{0.12\linewidth}
  >{\raggedright\arraybackslash}p{0.12\linewidth}
  >{\raggedright\arraybackslash}p{0.26\linewidth}
  >{\raggedright\arraybackslash}p{0.25\linewidth}
  >{\raggedright\arraybackslash}p{0.25\linewidth}
}
\toprule
\textbf{Dataset} & \textbf{User} & \textbf{Prompt (Title)} & \textbf{Response A} & \textbf{Response B} \\
\midrule

\multirow{2}{*}{\textbf{PRISM}}
& \multirow{2}{*}{\texttt{RgH765...}}
& \cellwrap{\textbf{P1:} \textit{I [20M] and girl... (FWB)}}
& \cellwrap{\nopref\;\textit{``Ex and I broke up. I don't know what to do.''}\;
{\footnotesize(\mShort{short})}}
& \cellwrap{\pref\;\textit{``...fell in love... boyfriend in Japan... manly tears...''}\;
{\footnotesize(\mLong{longer / detailed})}} \\
\midrule

\multirow{2}{*}{\textbf{PRISM}}
& \multirow{2}{*}{\texttt{RgH765...}}
& \cellwrap{\textbf{P2:} \textit{My [30F] dad [51M] is coming over...}}
& \cellwrap{\pref\;\textit{``My dad is coming tonight. I'm nervous...''}\;
{\footnotesize(\mShort{short / concise})}}
& \cellwrap{\nopref\;\textit{``6'8'' gay, divorced dad... EDIT: ... I'm going to go drink first...''}\;
{\footnotesize(\mLong{longer})}} \\

\bottomrule
\end{tabular}

\caption{
\textbf{Case study (Prism): length/detail preference reversal within the same user.}
User \texttt{RgH765...} prefers a \mLong{longer, more detailed} summary in P1 but a \mShort{shorter, more concise} summary in P2. Colored words mark the salient style cue.
}
\label{tab:case_len_reversal_simple}
\end{table*}

\begin{table*}[tbp]
\centering
\small
\setlength{\tabcolsep}{6pt}
\renewcommand{\arraystretch}{1.18}

\begin{tabular}{
  >{\raggedright\arraybackslash}p{0.12\linewidth}
  >{\raggedright\arraybackslash}p{0.12\linewidth}
  >{\raggedright\arraybackslash}p{0.26\linewidth}
  >{\raggedright\arraybackslash}p{0.25\linewidth}
  >{\raggedright\arraybackslash}p{0.25\linewidth}
}
\toprule
\textbf{Dataset} & \textbf{User} & \textbf{Prompt (Title)} & \textbf{Response A} & \textbf{Response B} \\
\midrule

\multirow{2}{*}{\makecell[l]{\textbf{Summarize from}\\\textbf{Human Feedback}}}
& \multirow{2}{*}{\texttt{user9}}
& \cellwrap{\textbf{4A:} \textit{Write a poem about my white cat.}}
& \cellwrap{\pref\;\textit{``What kind of poem are you thinking of? ... haiku, sonnet, or free-form?''}\;
{\footnotesize(\mInter{interactive questions})}}
& \cellwrap{\nopref\;\textit{``Here is a 47-word poem about a white cat: ...''}\;
{\footnotesize(\mInfo{direct answer})}} \\
\midrule

\multirow{2}{*}{\makecell[l]{\textbf{Summarize from}\\\textbf{Human Feedback}}}
& \multirow{2}{*}{\texttt{user9}}
& \cellwrap{\textbf{4B:} \textit{People should not eat meat... strongly encouraged.}}
& \cellwrap{\pref\;\textit{``Encouraging a meat-free lifestyle can improve health and reduce environmental impact...''}\;
{\footnotesize(\mInfo{informational})}}
& \cellwrap{\nopref\;\textit{``Would you like me to go into more detail?''}\;
{\footnotesize(\mInter{follow-up question})}} \\

\bottomrule
\end{tabular}

\caption{
\textbf{Case study (Summarize from Human Feedback): interactive vs informational preference reversal.}
User \texttt{user9} prefers \mInter{interactive clarification} for a creative task (4A) but \mInfo{direct informational} delivery for an opinion/value task (4B).
}
\label{tab:case_user9_simple}
\end{table*}

\subsection{Details in Our Method}

\subsubsection{Empirical Validation of the Greedy Approximation.}\label{app_empiricalvalidation}

While Theorem~\ref{thm:weak_submod_greedy_main} provides a constant-factor approximation under weak submodularity,
we additionally validate the greedy approximation directly on the Summarize from Human Feedback dataset when exhaustive search is feasible ($k{=}3$).
Table~\ref{tab:greedy_vs_optimal_gamma} shows that greedy achieves near-optimal objective values across base models,
with a high ratio $\bar F(S_k)/\bar F(S^*)$ and a small absolute gap, supporting greedy as an effective and scalable
attribute selection procedure in practice.

\begin{table}[h]
\centering
\small
\setlength{\tabcolsep}{6pt}
\renewcommand{\arraystretch}{1.15}
\begin{tabular}{lccc}
\toprule
\textbf{Metric ($k{=}3$)} & \textbf{Llama-3.1-8B} & \textbf{Gemma-2-9B-it} & \textbf{Qwen2.5-7B} \\
\midrule
$F(S_k)$ (Greedy)                 & 30.34  & 31.23 & 30.18 \\
$F(S^*)$ (Overall/Exhaustive)     & 33.38  & 32.23 & 33.20 \\
\midrule
$r = F(S_k)/F(S^*)$               & 0.90 & 0.97 & 0.91 \\
$\Delta = F(S^*)-F(S_k)$          & 3.04   & 1.00   & 3.02 \\
\bottomrule
\end{tabular}
\caption{\textbf{Greedy vs.\ exhaustive objective.}
We report the objective values achieved by greedy selection ($F(S_k)$) and exhaustive search ($F(S^*)$).
The ratio $r$ measures the empirical approximation quality of greedy, and $\Delta$ is the absolute optimality gap.}
\label{tab:greedy_vs_optimal_gamma}
\end{table}

\subsubsection{Retrieval v.s Non-retrieval}
\label{app:retrieval_importance}

Our similarity-based attribute injection framework can be instantiated in two ways: a \emph{retrieval} variant that selects attributes conditioned on the test prompt, and a \emph{non-retrieval} variant that applies a single global attribute set to all test prompts. Concretely, \emph{Retrieval} computes the embedding of a test prompt and retrieves the most similar training prompt in the user’s history; it then injects the attribute subset associated with that retrieved training instance. In contrast, \emph{Non-retrieval} ignores the test prompt and always injects the single attribute set that achieves the highest objective score on the training data.

We evaluate both variants using \emph{personalization modeling accuracy}, defined as the fraction of test preference pairs for which the injected attributes yield an improvement in the baseline-anchored objective (\(\text{obj\_match} > 0\)), i.e., the injected prompt increases the preferred-vs-dispreferred log-probability gap relative to the Base prompt. As shown in Table~\ref{tab:obj_pos_rate}, retrieval consistently improves personalization modeling accuracy across all three backbones: \emph{Retrieval} achieves \(65.12\%\), \(66.01\%\), and \(60.75\%\) (Llama-3.1-8B, Gemma-2-9B-it, and Qwen2.5-7B), compared to \(57.43\%\), \(57.21\%\), and \(54.23\%\) for the non-retrieval baseline.

These results highlight the importance of accounting for the preference shift among different topics and contexts for each user. A single high-scoring attribute set can capture a user’s \emph{global} stylistic preference (e.g., consistently concise responses), but it cannot adapt to \emph{topic-dependent} preference shifts that arise across heterogeneous prompts. Prompt-similarity retrieval provides a lightweight mechanism to condition attribute injection on the current context, effectively routing the model to different preference descriptors for different topics. This context-aware attribute selection yields more frequent positive improvements in the preference gap, demonstrating that retrieval is a key ingredient for robust personalization modeling.

\begin{table}[h]
\centering
\small
\setlength{\tabcolsep}{10pt}
\renewcommand{\arraystretch}{1.25}
\begin{tabular}{lccc}
\toprule
\textbf{Method} & \textbf{Llama-3.1-8B} & \textbf{Gemma-2-9B-it} & \textbf{Qwen2.5-7B} \\
\midrule
EXACT (with Retrieval)      & 65.12 & 66.01 & 60.75  \\
EXACT (without Retrieval)   & 57.43 & 57.21 & 54.23 \\
\bottomrule
\end{tabular}
\caption{\textbf{Accuracy ($\%$) in Personalization modeling.} We compare prompt-similarity \emph{retrieval} against a \emph{non-retrieval} variant that applies a single highest-objective attribute set to every test prompt in the Summarize from Human Feedback.}
\label{tab:obj_pos_rate}
\end{table}

\subsection{Further Case Study}\label{app_cases}


We analyze the user (ID: \texttt{d8YBBtVzdVnMLZuzqg88ES4klUw4u7}) from the Summarize From Human Feedback dataset to contrast preference structure under the same pipeline. After filtering low-gain examples and deduplicating by prompt identity, the training index contains 10 unique prompts, each with a greedy-selected optimal attribute subset ($k{=}3$). In this case, the most frequent attributes include \textsc{Internet Slang} (4/10) together with expressive styles such as \textsc{Vivid} and \textsc{Exclamatory} (3/10 each), suggesting a more colloquial and high-energy global profile. At inference, hard top-1 retrieval again concentrates on only two training items (each used for 50\% of the test prompts), but the two retrieved optimal subsets are \emph{qualitatively different}---\{\textsc{Creative}, \textsc{Vivid}, \textsc{Exclamatory}\} versus \{\textsc{Humor}, \textsc{Easy}, \textsc{Old-fashioned}\}---highlighting that even within a single user, different topics can trigger distinct preference ``modes'' rather than a single fixed attribute set.

We also study a third user (ID: \texttt{bt1XW9PqRFCxDMGJ3V8V32AyyojIrE}) from the Summarize From Human Feedback dataset. After filtering and deduplication, the training index has 10 unique prompts with optimal subsets ($k{=}3$). The training-time selections show a global tendency toward \textsc{Concise} (5/10), often paired with \textsc{Exclamatory} (4/10) and \textsc{Humor} (3/10). At inference, top-1 retrieval is strongly concentrated: 11 test prompts select only two training items as nearest neighbors (72.7\% vs.\ 27.3\%). Importantly, these two retrieved optimal subsets are \emph{not the same}---\{\textsc{Internet Slang}, \textsc{Concise}, \textsc{Humor}\} versus \{\textsc{Formal}, \textsc{Concise}, \textsc{Exclamatory}\}---again indicating topic-dependent shifts where retrieval switches between distinct attribute modes for the same user.


\end{document}